\RequirePackage{fix-cm}
\documentclass[smallcondensed]{svjour3}
\smartqed

\usepackage{balance}
\usepackage[bookmarks=false,colorlinks=true,citecolor=Blue,linkcolor=Blue]{hyperref}
\usepackage{booktabs}
\usepackage{nicefrac}
\usepackage[usenames,dvipsnames]{xcolor}

\usepackage{mathtools,bbm}
\usepackage{mathrsfs}
\usepackage{amssymb}
\usepackage{amsmath}
\usepackage{enumitem}
\usepackage{graphicx,subfigure}

\def \R{\mathbb{R}}
\def \Nat{\mathbb{N}}
\def \Pr{\mathbb{P}}
\def \Gr{\mathcal{G}}
\def \Dt{\mathcal{D}_n}
\def \RKHS{\mathcal{H}}
\def \Ker{\mathrm{K}_x}
\def \Reg{\mathrm{\Phi}}

\def \err{\widehat{\varepsilon}}
\def \one{\mathbbm{1}}

\newtheorem{assumption}{Assumption}

\newcommand{\argmin}{\operatornamewithlimits{argmin}}
\newcommand{\tr}{^\mathrm{T}}

\AtBeginDocument{%
  \paperwidth=\dimexpr
    1in + \oddsidemargin
    + \textwidth
    + 1in + \oddsidemargin
  \relax
  \paperheight=\dimexpr
    1in + \topmargin
    + \headheight + \headsep
    + \textheight
    + 1in + \topmargin
  \relax
  \usepackage[pass]{geometry}\relax
}

\usepackage{natbib}

\begin{document}

\title{Distribution-Free Uncertainty Quantification\\ for Kernel Methods by Gradient Perturbations}

\titlerunning{Distribution-Free Uncertainty Quantification for Kernel Methods}

\author{Bal\'azs~Csan\'ad~Cs\'aji         \and
        Kriszti\'an~Bal\'azs~Kis }

\institute{Bal\'azs~Csan\'ad~Cs\'aji \at
			EPIC Centre of Excellence\\
			MTA SZTAKI: Institute for Computer Science and Control\\
			Hungarian Academy of Sciences, Budapest, Hungary\\
              Tel.: +(36)-1-279-6231\quad Fax: +(36)-1-279-7503\\
              \email{balazs.csaji@sztaki.mta.hu} \and Kriszti\'an~Bal\'azs~Kis  \at EPIC Centre of Excellence\\
           MTA SZTAKI: Institute for Computer Science and Control\\
           Hungarian Academy of Sciences, Budapest, Hungary\\
           Tel.: +(36)-1-279-6111\quad Fax: +(36)-1-279-7503\\
           \email{krisztian.kis@sztaki.mta.hu} }

\date{}

\maketitle

\begin{abstract}
We propose a data-driven approach to quantify the uncertainty of models constructed by kernel methods. Our approach  minimizes the needed distributional assumptions, hence, instead of working with, for example, Gaussian processes or exponential families, it only requires knowledge about some mild regularity of the measurement noise, such as it is being symmetric or exchangeable. We show, by building on recent results from finite-sample system identification, that by perturbing the residuals in the gradient of the objective function, information can be extracted about the amount of uncertainty our model has. Particularly, we provide an algorithm to build exact, non-asymptotically guaranteed, distribution-free confidence regions for ideal, noise-free representations of the function we try to estimate. For the typical convex quadratic problems and symmetric noises, the regions are star convex centered around a given nominal estimate, and have efficient ellipsoidal outer approximations. Finally, we illustrate the ideas on typical kernel methods, such as LS-SVC, KRR, $\varepsilon$-SVR and kernelized LASSO.
\end{abstract}

\newpage
\section{Introduction}
\label{intro}
Kernel methods build on the fundamental concept of Reproducing Kernel Hilbert Spaces \citep{aronszajn1950theory,gine2015mathematical} and are widely used in machine learning \citep{shawe2004kernel,hofmann2008kernel} and related fields, such as system identification \citep{pillonetto2014kernel}. One of the reasons for their popularity is the representer theorem \citep{kimeldorf1971some,scholkopf2001generalized} which shows that finding an estimate in an infinite-dimensional space of functions can be traced back to a finite-dimensional problem. Support vector machines \citep{scholkopf2001learning,steinwart2008support}, rooted in statistical learning theory \citep{Vapnik1998}, are typical examples of kernel methods.

Besides how to construct efficient models from data, it is also a fundamental question how to quantify the {\em uncertainty} of the obtained models. While standard approaches like Gaussian processes \citep{Rasmussen2006} or exponential families \citep{hofmann2008kernel} offer a nice theoretical framework, making strong statistical assumptions on the system is sometimes unrealistic, since in practice we typically have very limited knowledge about the noise affecting the measurements. Building on asymptotic results, such as limiting distributions, is also widespread \citep{gine2015mathematical}, but they usually lack finite-sample guarantees.

Here, we propose a {\em non-asymptotic}, {\em distribution-free} approach to quantify the uncertainty of kernel-based models, which can be used for {\em hypothesis testing} and {\em confidence region} constructions. We build on recent developments in finite-sample system identification \citep{campi2005guaranteed,Algo2018}, more specifically, we build on the Sign-Perturbed Sums (SPS) algorithm \citep{SPSPaper2ITA} and its generalizations, the Data Perturbation (DP) methods \citep{KolumbanThesis2016}.

We consider the case where there is an underlying ``true'' function that generates the measurements, but we only have noisy observations of its outputs. Since we want to minimize the needed assumptions, for example, we do not want to assume that the true underlying function belongs to the Hilbert space in which we search our estimate, we take an ``honest'' approach \citep{li1989honest} and consider ``ideal'' representations of the target function from our function space. A representation is ideal w.r.t.\ the data sample, if its outputs coincide with the corresponding (hidden) noise-free outputs of the true underlying function for all available inputs.

Although our method is {\em distribution-free}, i.e., it does not depend on any parameterized distributions, it has strong {\em finite-sample guarantees}. We show that the constructed confidence region contains the ideal representation {\em exactly} with a user-chosen probability. In case the noises are independent and symmetric about zero, and the objective function is convex quadratic, the resulting regions are {\em star convex} and have efficient {\em ellipsoidal outer approximations}, which can be computed by solving 
semidefinite optimization problems. Finally, we demonstrate our approach on typical kernel methods, such as KRR, SVMs and kernelized LASSO.

Our approach has some similarities to bootstrap \citep{efron1994introduction} and conformal prediction \citep{vovk2005algorithmic}. One of the fundamental differences w.r.t.\ bootstrap is, e.g., that we avoid building alternative samples and fitting bootstrap estimates to them (since it is computationally challenging), but perturb directly the gradient of the objective function. Key differences w.r.t.\ conformal prediction are, e.g., that we want to quantify the uncertainty of the model and not necessarily that of the next observation (though the two problems are related), and more importantly, exchangeability is not fundamental for our approach.

\section{Preliminaries}
A Hilbert space, $\mathcal{H}$, of functions $f: \mathcal{X} \to \mathbb{R}$, with inner product $\left< \cdot, \cdot \right>_{\mathcal{H}}$, is called a {\em Reproducing Kernel Hilbert Space} (RKHS), if the point evaluation functional
\begin{equation}
 \delta_z : \RKHS \to \R, \qquad f \mapsto f(z),
\end{equation}
is continuous (or equivalently bounded) for all $z \in \mathcal{X}$, at any $f \in \mathcal{H}$ \citep{gine2015mathematical}. Then, by using the Riesz representation theorem, one can construct a (unique) kernel, $k : \mathcal{X} \times \mathcal{X} \to \mathbb{R}$, having the {\em reproducing} property, that is
\begin{equation}
\left<\hspace{0.2mm} k(\cdot, z), f \hspace{0.2mm}\right>_{\mathcal{H}} \,=\, f(z),
\end{equation}
for all $z \in \mathcal{X}$ and $f \in \mathcal{H}$. In particular, the kernel satisfies for all $z,s \in \mathcal{X}$ that
\begin{equation}
k(z, s) \,=\, \left<\hspace{0.2mm} k(\cdot, z), k(\cdot, s) \hspace{0.2mm}\right>_{\mathcal{H}}.
\end{equation}
Hence, the kernel of an RKHS is a symmetric and positive-definite function; moreover, the Moore--Aronszajn theorem states that the converse is also true: for every symmetric, positive-definite function there is a unique RKHS \citep{aronszajn1950theory}.

Typical kernels include, e.g., the Gaussian kernel $k(z, s) = \exp(\nicefrac{-\|z-s\|^2}{2 \sigma^2})$, with $\sigma > 0$, the polynomial kernel, $k(z, s) = (\left< z, s \right> + c)^p$, with $c \geq 0$ and $p \in \mathbb{N}$, and the sigmoidal kernel, $k(z, s) = \tanh(a \left< z, s \right> + b)$ for some $a, b \geq 0$,  where $\left< \cdot, \cdot \right>$ denotes the standard Euclidean inner product \citep{hofmann2008kernel}.

By a {\em data sample}, $\Dt$, we mean a finite set of input-output measurements,
\begin{equation}
(x_1, y_1), \,\dots,\, (x_n, y_n)\, \in\, \mathcal{X} \times \mathbb{R},
\end{equation}
with $\mathcal{X} \neq \emptyset$. We also introduce $x \doteq (x_1, \dots, x_n)\tr \in \mathcal{X}^n$ and $y \doteq (y_1, \dots, y_n)\tr  \in \mathbb{R}^n$. The {\em Gram matrix} of $k(\cdot, \cdot)$, w.r.t.\ input $x$, is denoted by $\Ker \in \R^{n\times n}$, where
\begin{equation}
[\,\Ker\,]_{i,j}\, \doteq \,k(x_i, x_j).
\end{equation}
A kernel is called {\em strictly} positive definite if its Gram matrix, $\Ker$, is (strictly) positive definite for {\em distinct} inputs $\{x_i\}$ \citep{hofmann2008kernel}.

One of the fundamental reasons for the successes of kernel methods is the so-called {\em representer theorem}, originally given by \cite{kimeldorf1971some}, but the generalization presented here is due to \cite{scholkopf2001generalized}.

\begin{theorem}
Suppose we are given a sample, $\Dt$, a positive-definite kernel $k(\cdot, \cdot)$, an associated RKHS with a norm $\|\cdot\|_{\mathcal{H}}$ induced by $\left< \cdot, \cdot \right>_{\mathcal{H}}$, and a class of functions
\begin{align}
\mathcal{F}\,\, \doteq\, & \,\,\Big\{\, f: \mathcal{X} \to \mathbb{R}\, \mid\, f(z) = \sum_{i=1}^{\infty} \beta_i k(z, z_i),\,
\beta_i \in \R,\, z_i \in \mathcal{X},\, \|f\|_{\mathcal{H}} < \infty \, \Big\},
\end{align}
then, for any monotonically increasing regularization function, $\mathrm{\Lambda}: [0, \infty) \to [0, \infty)$, and an arbitrary loss function $\mathrm{L}: (\mathcal{X} \times \R ^2)^n \to \R \cup \{ \infty \}$, if the objective
\begin{equation}
g(f, \Dt) \; \doteq \; \mathrm{L}\hspace{0.2mm}\big( \,(x_1, y_1, f(x_1)), \dots, (x_n, y_n, f(x_n))\, \big) \,+\, \mathrm{\Lambda}\hspace{0.2mm}(\,\|f\|_{\mathcal{H}}\,),
\end{equation}
has a minimizer over $\mathcal{F}$, there is also a minimizer with the following representation\vspace{-1mm}
\begin{equation}
\label{kernel-alpha}
f_{\alpha}(z) \; = \; \sum_{i=1}^{n} \alpha_i\hspace{0.4mm} k(z, x_i),
\vspace{-1.5mm}
\end{equation}
where $\alpha \doteq (\alpha_1, \dots, \alpha_n)\tr \in \R^n$ is the vector of coefficients. If $\mathrm{\Lambda}$ is strictly monotonically increasing, each minimizer over $\mathcal{F}$ has a representation of the form \eqref{kernel-alpha}.
\end{theorem}

The theorem can be extended with a bias term \citep{scholkopf2001learning}, in which case if the solution exists, it also contains a multiple of the bias term. For further generalizations, see \citet{yu2013characterizing,argyriou2014unifying}.

The power of the representer theorem comes from the fact that it shows that computing the point estimate in a high, typically infinite, dimensional space of models can be reduced to a much simpler (finite-dimensional) optimization problem whose dimension does not exceed the size of the sample we have, that is $n$.

If the data is noisy, then of course, the obtained estimate is a {\em random} function and it is of natural interest to study the {\em distribution} of the resulting function, for example, to evaluate its {\em uncertainty} or to test {\em hypotheses} about the system.

\section{Confidence Regions for Kernel Methods}

Now, we turn our attention to a stochastic variant of the problem discussed above. There are several advantages of taking a statistical point of view on kernel methods, including conditional modeling, dealing with structured responses, handling missing measurements and building prediction regions \citep{hofmann2008kernel}.

Following a standard statistical viewpoint \citep{davies2009nonparametric}, we assume that the outputs $\{y_i\}$ are generated by some noisy observations of an underlying ``true'' function, denoted by $f_*$, that is for all $i = 1, \dots, n$, the outputs can be written as
\begin{equation}
y_i \; \doteq \; f_*(x_i)\, +\, \varepsilon_i,
\end{equation}
where $\{\varepsilon_i\}$ are the noise terms. The entire noise vector is $\varepsilon \,\doteq\, (\varepsilon_1, \dots, \varepsilon_n)\tr$. The noiseless outputs of  function $f_*$ will be denoted by $y_i^* \doteq f_*(x_i)$, for $i = 1,\dots,n$.

\subsection{Ideal Representations}
\label{sec:IdealReps}
We aim at {\em quantifying the uncertainty} of our estimated model. A standard way to measure the quality of a point-estimate is to build {\em  confidence regions} around it. However, it is not obvious what we should aim for with our confidence regions. For example, since all of our models live in our RKHS, $\RKHS$, we would like to treat the confidence region as a subset of $\RKHS$. On the other hand, we want to minimize the assumptions, for example, we may not want to assume that $f_*$ is an element of $\RKHS$. Furthermore, unless we make strong smoothness assumptions on the underlying unobserved function, we only have information about it at the actual inputs, $\{x_i\}$. Hence, we aim for an ``honest'' nonparametric approach \citep{li1989honest} and search for functions which correctly describe the hidden function, $f_*$, on the given inputs. Then, by the representer theorem, we may restrict ourselves to a finite-dimensional subspace of $\RKHS$. This leads us to the definition of {\em ideal representations}:

\begin{definition}
Let $\RKHS_{\alpha} \subseteq \RKHS$ denote the subspace of functions that can be represented as \eqref{kernel-alpha}. A function $f_0 \in \RKHS_{\alpha}$, having coefficients $\alpha_0 \in \R^n$, is called an {\em ideal} or noise-free representation of the ``true'' unobserved function $f_*$, if we have
\begin{equation}
f_{0}(x_i) \; =\; y_i^* \; \doteq \; f_*(x_i),\qquad\mbox{for all} \qquad i \in \{ \, 1, \dots, n \,\}.
\label{eq:idealrepdef}
\end{equation}
The set of all ideal representations, w.r.t.\ data sample $\Dt$, is denoted by $\mathcal{H}_{0} \subseteq \mathcal{H}_{\alpha}$, and the set of their coefficients, called {\em ideal coefficients}, is denoted by $A_0 \subseteq \R^n$.
\end{definition}

An ideal representation does not simply interpolate the observed (noisy) outputs $\{y_i\}$, but it interpolates the {\em unobserved} (noise-free) outputs, that is $\{y^*_i\}$.

A natural question which arises is: when does such an ideal representation exist? To answer this question, first note that since ideal representations have the form \eqref{kernel-alpha}, equation system \eqref{eq:idealrepdef} can be rewritten in a matrix form by using the Gram matrix. That is, vector $\alpha$ is an ideal coefficient vector, if and only if
\begin{equation}
\Ker\hspace{0.2mm} \alpha \; = \; y^*,
\label{eq:idealrepdef2}
\end{equation}
where $y^* \,\doteq\, (y^*_1, \dots, y^*_n)\tr$. If $\Ker$ is (strictly) positive definite, which is the case if for example the kernel is Gaussian and all inputs are distinct, then $\mbox{rank}(\Ker) = n$ and {\em every} $f_*: \mathcal{X} \to \R$ has a {\em unique} ideal representation w.r.t.\ data sample $\Dt$.

On the other hand, if $\mbox{rank}(\Ker) < n$, then \eqref{eq:idealrepdef2} places a restriction on the functions which have ideal representations. For example, if $\mathcal{X} = \R$  and $k(z,s) = \left<z,s\right> =  z\tr s$, then $\mbox{rank}(\Ker) = 1$ and in general only functions which are {\em linear} on the data sample have ideal representations. This is of course not surprising, as it is well-known that the choice of the {\em kernel} encodes our {\em inductive bias} on the underlying true function we aim at estimating \citep{scholkopf2001learning}.

If $\mbox{rank}(\Ker) < n$ and there is an $\alpha$ which satisfies \eqref{eq:idealrepdef2}, then there are infinitely many ideal representations, as for all $\nu \in \mbox{null}(\Ker)$, the null space of $\Ker$, we have $ \Ker\hspace{0.2mm} (\alpha + \nu) \, = \, \Ker\hspace{0.2mm} \alpha\, +\,\Ker\hspace{0.2mm} \nu \, =\, \Ker\hspace{0.2mm} \alpha \,=\, y^*. $ The opposite is also true, if $\alpha$ and $\beta$ both satisfy \eqref{eq:idealrepdef2}, then $ \Ker\hspace{0.2mm} (\alpha - \beta) \, =\, \Ker\hspace{0.2mm} \alpha\, -\,\Ker\hspace{0.2mm}\beta \,=\,0$, thus, $\alpha - \beta \in \mbox{null}(\Ker)$. Hence, to avoid allowing infinitely many ideal representations, we may form {\em equivalence classes} by treating coefficient vectors $\alpha$ and $\beta$ equivalent if\, $\Ker\hspace{0.2mm} \alpha\, = \,\Ker\hspace{0.2mm}\beta$. Then, we can work with the resulting {\em quotient space} of coefficients to ensure that there is only one ideal representation (i.e., one equivalence class of such representations).

All of our theory goes through if we work with the quotient space of representations, but to simplify the presentation we make the assumption (cf.\ Section \ref{sec:assumptions}) that $\Ker$ is full rank, therefore, there always {\em uniquely exists} an ideal representation (for any ``true'' function), whose unique coefficient vector will be denoted by $\alpha^*$.

\subsection{Exact and Honest Confidence Regions}
Let $(\hspace{0.2mm}\mathrm{\Omega}, \mathcal{A}, \{ \Pr_{\theta} \}_{\theta \in \mathrm{\Theta}})$ be a {\em statistical space}, where $\mathrm{\Theta}$ denotes an arbitrary {\em index set}. In other words, for all $\theta \in \mathrm{\Theta}$, $(\hspace{0.2mm}\mathrm{\Omega}, \mathcal{A}, \Pr_{\theta} \hspace{0.1mm})$ is a probability space, where $\mathrm{\Omega}$ is the sample space, $\mathcal{A}$ is the $\sigma$-algebra of events, and $\Pr_{\theta}$ is a probability measure. Note that it is {\em not assumed} that $\mathrm{\Theta} \subseteq \R^d$, for some $d$; therefore, this formulation covers {\em nonparametric} inference, as well (and that is why we do not call $\theta$ a ``parameter'').

In our case, index $\theta$ identifies the distribution generating the data, in particular, the underlying true function and the distribution of the noise, thus, each possible $\theta$ induces a different  distribution according to which the observations are generated. Confidence regions constitute a classical form of statistical inference, when we aim at constructing sets which cover with high probability some target function of $\theta$ \citep{degroot2012probability}. These sets are usually random as they are typically built using the data. We will build confidence regions for the ideal coefficient vector (equivalently, the ideal representation), which itself is a  random element.

Let $\gamma$ be a random element (it corresponds to the available observations), let $g(\theta, \gamma)$ be some target function of $\theta$ (which can possibly also depend on the observations) and let $p \in [\hspace{0.3mm}0,1\hspace{0.3mm}]$ be a target probability, also called confidence level. A confidence region for $g(\theta, \gamma)$ is a random set, $C(p, \gamma)$, which is a subset of the codomain of function $g$. The following definition formalizes two important types of stochastic guarantees for confidence regions \citep{davies2009nonparametric}.

\begin{definition}
A confidence region $ C(p, \gamma)$ for $g(\theta, \gamma)$ is called {\em exact}, if
\begin{equation}
\forall\,\theta \in \mathrm{\Theta}: \Pr_{\theta}\! \left(\,g(\theta, \gamma) \in C(p, \gamma)\,\right)\, = \; p,
\end{equation}
and it is called {\em honest}, if it satisfies\, $\forall\,\theta \in \mathrm{\Theta}: \Pr_{\theta}\! \left(\,g(\theta, \gamma) \in C(p, \gamma)\,\right)\, \geq \; p.$
\end{definition}

In our case, $\gamma$ is basically\footnote{We used the word ``basically'', since there will also be some other random elements in the construction, e.g., for tie-breaking, and those should also constitute part of observation $\gamma$.} the sample of input-output pairs, $\Dt$; and the target object we aim at covering is $g(\theta, \gamma) = \alpha^*_{\theta}$, i.e., the (unique) ideal coefficient vector corresponding to the underlying true function (identified by $\theta$) and the sample. Since the ideal coefficient vector uniquely determines the ideal representation (together with the inputs, which however we observe), it is enough to estimate the former. The main question of this paper is how we can construct exact or honest confidence regions for the ideal coefficient vector based on a finite sample without strong distributional assumptions on the statistical space.

Henceforth, we will treat $\theta$ (the underlying true function) fixed, and omit the $\theta$ indexes from the notations, to simplify the formulas. Therefore, instead of writing $\Pr_{\theta}$ or $\alpha^*_{\theta}$, we will simply use $\Pr$ or $\alpha^*$. The results are of course valid for all $\theta$.

Standard ways to construct confidence regions for kernel-based estimates typically either make {\em strong distributional assumptions}, like assuming Gaussian processes \citep{Rasmussen2006}, or resort to {\em asymptotic} results, such as Donsker-type theorems for Kolmogorov--Smirnov confidence bands. An alternative approach is to build on {\em Rademacher complexities}, which can provide non-asymptotic, distribution-free confidence bands \citep{gine2015mathematical}. Nevertheless, these regions are conservative (not exact) and are constructed independently of the applied kernel method. In contrast, our approach provides {\em exact}, {\em non-asymptotic}, {\em distribution-free} confidence sets for a {\em user-chosen} kernel estimate.

\section{Non-Asymptotic, Distribution-Free Framework}
\label{dfr}
This section presents the proposed framework to quantify the uncertainty of kernel-based estimates. It is inspired by and builds on recent results from finite-sample system identification, such as the SPS and DP methods \citep{campi2005guaranteed,SPSPaper2ITA,csaji2016score,KolumbanThesis2016,Algo2018}. Novelties with respect to these approaches are, e.g., that our framework considers {\em nonparametric} regression and does not require the ``true'' function to be in the model class.

\subsection{Distributional Invariance}
\label{sec:distinv}
The proposed method is distribution-free in the sense that it does not presuppose any parametric distribution about the noise vector $\varepsilon$. We only assume some mild regularity about the measurement noises, more precisely that their (joint) distribution is invariant with respect to a known group of transformations.

\begin{definition}
\label{def-dist-invar}
An $\R^n$-valued random vector $v$ is {\em distributionally invariant} with respect to a compact {\em group of transformations}, $(\Gr, \circ)$, where ``$\circ$'' is the function composition and each $G \in \Gr$ maps $\R^n$ to itself, if for all transformations $G \in \Gr$, random vectors $v$ and $G(v)$ have the same distribution.
\end{definition}

The two most important examples of the above definition are as follows.
\smallskip
\begin{itemize}
\item If $\{\varepsilon_i\}$ are {\em exchangeable} random variables, then the (joint) distribution of the noise vector $\varepsilon$ is invariant w.r.t.\ multiplications by permutation matrices (which are orthogonal and form a finite, thus compact, group).
\medskip

\item On the other hand, if $\{\varepsilon_i\}$ are independent, each having a (possibly different!) {\em symmetric} distribution about zero, then the (joint) distribution of $\varepsilon$ is invariant w.r.t.\ multiplications by diagonal matrices having $+1$ or $-1$ as diagonal elements (which are also orthogonal, and form a finite group).
\end{itemize}
\smallskip

Both of these examples assume only mild regularities about the measurement noises: for example, it is a standard assumption in statistical learning theory that the sample is independent and identically distributed (i.i.d.) which immediately implies exchangeability (which is a more general concept than i.i.d.). But even this assumption can be omitted if we work with symmetric noises, which are widespread as most standard distributions in statistics are symmetric, such as Gauss, Laplace, Cauchy, Student's t, uniform, plus a large class of multimodal ones.

Note that for these examples no assumptions about other properties of the (noise) distributions are needed, e.g., they can be {\em heavy-tailed}, with even {\em infinite variance}, skewed (e.g., in the exchangeable case), and {\em no moment assumptions} are necessary. For the case of symmetric distributions, it is even allowed that the observations are affected by a noise where each $\varepsilon_i$ has a {\em different} distribution.

\subsection{Main Assumptions}
\label{sec:assumptions}
Before the general construction of our method is explained, first, we highlight the core assumptions we apply. We also discuss their relevance and implications.

\begin{assumption}
\label{A0}
The kernel, $k(\cdot, \cdot)$, is strictly positive definite and all inputs, $\{x_i\}$, are distinct with probability one {\em(}\,in other words, $\forall\,i\neq j: \Pr\hspace{0.2mm}(\hspace{0.2mm}x_i = x_j\hspace{0.2mm}) \,= \,0$\,{\em)}.
\end{assumption}
As we discussed in Section \ref{sec:IdealReps}, this assumption ensures that $\mbox{rank}(\Ker) = n$ (a.s.), hence there {\em uniquely exists} an ideal representation (a.s.), whose unique ideal coefficient vector is denoted by $\alpha^*$. The primary choices are {\em universal} kernels for which $\RKHS$ is dense in the space of continuous functions on compact domains of $\mathcal{X}$.

\begin{assumption}
\label{A1}
The input vector $x$ and the noise vector $\varepsilon$ are independent.
\end{assumption}

Assumption \ref{A1} implies that the measurement noises, $\{\varepsilon_i\}$, do not affect the inputs, $\{x_i\}$; for example, the system is not autoregressive. It is possible to extend our approach to dynamical systems, e.g., using similar ideas as in \citep{Csaji2012b,csaji2015closed,csaji2016score},  but we leave the extension for future research. Note that Assumption \ref{A1} allows deterministic inputs, as a special case.

\begin{assumption}
\label{A2}
Noise $\varepsilon$ is distributionally invariant w.r.t.\ a known group of transformations, $(\Gr, \circ)$, where each $G \in \Gr$ acts on $\R^n$ and $\circ$ is the function composition.
\end{assumption}

Assumption \ref{A2} states that we know transformations that do not change the (joint) distribution of the measurement noises. As it was discussed in Section \ref{sec:distinv}, symmetry and exchangeability are two standard  examples for which we know such group of transformations. Thus, if the noise vector is either exchangeable (e.g., it is i.i.d.), or symmetric, or both properties hold, then the theory applies. We also note that the suggested methodology is not limited to exchangeable or symmetric noises, e.g.,  power defined noises constitute another example \citep{KolumbanThesis2016}.

\begin{assumption}
\label{A3}
The gradient, or a subgradient, of the objective w.r.t. $\alpha$ exists and it only depends on the output vector, $y$, through the residuals, i.e., there is $\bar{g}$,
\begin{equation}
(\nabla_{\!\alpha}\, g)({f}_{\alpha}, \Dt )\; = \; \bar{g}(x, \alpha, \err(x, y, \alpha) ),
\end{equation}
where the residuals w.r.t.\ the sample and the coefficients are defined as
\begin{equation}
\err(x, y, \alpha) \;\doteq\; y\, -\, \Ker\hspace{0.2mm} \alpha.
\end{equation}
\end{assumption}

For Assumption \ref{A3}, it is enough if a subgradient is defined for each coefficient vector $\alpha$, hence, e.g., the cases of {\em $\varepsilon$-insensitive} and {\em Huber} loss functions are also covered. Even in such cases (when we work with subderivatives), we still treat $\bar{g}$ as a vector-valued function and choose arbitrarily from the set of possible subgradients.

This requirement is also very mild as it is typically  the case that the objective function is differentiable or convex and has subgradients (we will present several demonstrative examples in Section \ref{sec-examples}); furthermore, the objective typically only depends on $y$ through the residuals, which immediately imply Assumption $\ref{A3}$.

To see this assume that $g$ is differentiable; then clearly, if the objective function can be written as $g({f}_{\alpha}, \Dt )\, = \, g_0(x, \alpha, \err(x, y, \alpha) )$ for some function $g_0$, then
\vspace{-1mm}
\begin{align}
(\nabla_{\!\alpha}\, g)({f}_{\alpha}, \Dt )\, = & \,\,\,(\nabla_{\!\alpha}\, g_0)(x, \alpha,  y - \Ker\hspace{0.2mm}\alpha) \nonumber \\[0.5mm]
=& \,\,\, (\nabla_{\!2}\, g_0) (x, \alpha,  y - \Ker\hspace{0.2mm}\alpha) \,-\, \Ker (\nabla_{\!3}\, g_0) (x, \alpha,  y - \Ker\hspace{0.2mm}\alpha) \nonumber \\[0.5mm]
= & \,\,\, \bar{g}(x, \alpha, \err(x, y, \alpha) ),\\[-6mm]\nonumber
\end{align}
where $\nabla_{\!2}\, g_0$ and $\nabla_{\!3}\, g_0$ denote the gradients of $g_0$ w.r.t.\ its second and third arguments, respectively; and during the derivation we used the chain rule, as well as the fact that $\Ker$ is symmetric and the definition of the residuals, $\err(x, y, \alpha)\, =\, y - \Ker\hspace{0.2mm} \alpha$.

\subsection{Perturbed Gradients}
At first, the proposed method can be understood as a {\em hypothesis testing} approach. Given coefficient vector $\alpha \in \R^n$ we test the null hypothesis $H_0: \alpha \, = \, \alpha^*$, i.e., it is the ideal coefficient vector, against the alternative hypothesis $H_1: \alpha\, \neq\, \alpha^*$.  Under $H_0$, the residuals of $f_{\alpha}$ coincide with the ``true'' (unobserved) noise terms, since by definition (for ideal representations), we have
\vspace{-1mm}
\begin{align}
\err(x, y, \alpha^*) \,=\, &\,\,\,  y \,-\, \Ker\hspace{0.2mm} \alpha^* \nonumber  \\[0.5mm]
= \, &\,\,\, [\,f_*(x_1) + \varepsilon_1, \dots, f_*(x_n) + \varepsilon_n\,]\tr \nonumber \\[0.5mm]
- \, & \, \, \, [\,f_*(x_1), \dots, f_*(x_n)\,]\tr \,\,=\,\, \varepsilon.\\[-6mm]\nonumber
\end{align}
Consequently, based on the group of invariant transformations, $\Gr$, we know that the (joint) distribution of the residuals does not change if we transform them by any $G \in \Gr$ (under $H_0$). Then, we can generate alternative realizations of the residuals, $\err(x, y, \alpha^*)$, by applying a random transformation $G \in \Gr$, and the resulting alternative realization, $G(\err(x, y, \alpha^*))$, will behave ``similarly'' (in the statistical sense) to the original residual vector (i.e., the true noise vector).

However, under $H_1$, if coefficient vector $\alpha$ does not define an ideal representation, $\err(x, y, \alpha)$, in general, will not coincide with the true noises. Therefore, the distributions of their randomly transformed variants will be distorted and will statistically {\em not} behave ``similarly'' to the original residuals.

Of course, we need a way to measure ``similar behavior''. Since we want to measure the uncertainty of a model constructed by using a certain objective function, we will measure similarity by recalculating (the magnitude of) its gradient (w.r.t.\ $\alpha$) with the transformed residuals and apply a rank test \citep{good2005permutation}.

Let us define a {\em reference} function, $Z_0 : \R^n \to \R$, and $m-1$ {\em perturbed} functions, $\{Z_i\}$, with $Z_i: \R^n \to \R$, where $m$ is a user-chosen hyper-parameter, as follows
\vspace{-1mm}
\begin{align}
\label{refpert}
Z_0(\alpha) \, &\doteq \,  \|\, \mathrm{\Psi}(x) \, \bar{g}(x, \alpha, G_0 (\err(x, y, \alpha)) ) \, \|^2,\\[2mm]
\label{refpert2}
Z_i(\alpha) \, &\doteq \,  \|\, \mathrm{\Psi}(x) \, \bar{g}(x, \alpha, G_i (\err(x, y, \alpha)) ) \, \|^2,
\end{align}
for $i= 1, \dots, m-1$, where $\mathrm{\Psi}(x)$ is some (possibly input dependent) positive definite weighting matrix, $G_0$ is the {\em identity} element of $\Gr$ (w.l.o.g.\ the identity transformation), and $\{G_i\}$ are i.i.d.\ random transformations from $\Gr$, sampled using the {\em uniform} distribution on $\Gr$. They are generated independently of the other random elements of the system, such as the input vector $x$ and the noise vector $\varepsilon$.

For symmetric noises, transformation $G_i \in \Gr$ is basically a random $n \times n$ diagonal matrix whose diagonal elements are $+1$ or $-1$, each having $\nicefrac{1}{2}$ probability to be selected, independently of the other elements of the diagonal.

On the other hand, for the case of exchangeable noise terms, each transformation $G_i \in \Gr$ is a randomly (uniformly) chosen $n \times n$ permutation matrix.

Weighting matrix $\mathrm{\Psi}(x)$ is included in the construction to allow some additional flexibility, e.g., if we have some a priori information on the measurement noises. We will see an example for the special case of quadratic objectives in Section \ref{subsec:quadratic}. In case no such information is available, $\mathrm{\Psi}(x)$ can be chosen as identity.

We can observe that for the ideal coefficient vector $\alpha^*$, we have
\begin{align}
Z_0(\alpha^*)\, \, = & \,\,\,  \|\, \mathrm{\Psi}(x) \,  \bar{g}(x, \alpha^*, \varepsilon ) \, \|^2 \nonumber \\[1mm]
{\buildrel d \over =} & \,\,\, \|\, \mathrm{\Psi}(x) \, \bar{g}(x, \alpha^*, G_i (\varepsilon) ) \, \|^2 \nonumber  \\[1.5mm]
=  & \,\,\,  Z_i(\alpha^*),\\[-7mm]
\nonumber
\end{align}
for $i= 1, \dots, m-1$, where ``${\buildrel d \over =}$'' denotes equality in distribution. Therefore, the $\{Z_i(\alpha^*)\}_{i=0}^{m-1}$ variables have the same (marginal) distribution, though, they are of course not independent. It can be shown, however, that they are {\em conditionally} i.i.d.\ (hence, exchangeable), and therefore all of their possible orderings are equally likely, with possible tie-breakings, which can be used to measure {\em similar} behavior.

On the other hand, for $\alpha \neq \alpha^*$, this distributional equivalence does not hold, and we expect that if\, $\|\hspace{0.3mm} \alpha - \alpha^*\hspace{0.3mm}  \|$ is large enough, the reference element $Z_0(\alpha)$ will dominate the perturbed elements, $\{Z_i(\alpha)\}_{i=1}^{m-1}$, with high probability, from which we can detect (statistically) that coefficient vector $\alpha$ is not the ideal one, $\alpha \neq \alpha^*$.

\subsection{Normalized Ranks}
Now, we make our argument, including possible tie-breakings, more precise by introducing the concept of normalized ranks. Formally, the {\em normalized rank} of the reference element, $Z_0(\alpha)$, among all $\{Z_i(\alpha)\}_{i=0}^{m-1}$ elements is defined as follows
\begin{equation}
\mathcal{R}(\alpha)\, \doteq \,\mathcal{R}_m(\alpha)\,  \doteq\, \frac{1}{m} \bigg[\, 1 + \sum_{i=1}^{m-1}
\mathbb{I}\left( Z_i(\alpha) \prec_{\pi} Z_0(\alpha) \right) \bigg],
\end{equation}
where $\mathbb{I}(\cdot)$ is an indicator function, namely, its value is $1$ if its argument is true and $0$ otherwise; $m \in \Nat$ is a user-chosen hyper-parameter; and binary relation ``$\prec_{\pi}$'' is the standard ``$<$'' with random tie-breaking (according to a fixed, pre-generated random order). More precisely, let $\pi$ be a random (uniformly chosen) permutation of the set $\{ 0, \dots, m-1 \}$. Then, given $m$ arbitrary real numbers, $Z_0, \dots, Z_{m-1}$, we can construct a strict total order, denoted by ``$\prec_{\pi}$'', by defining $Z_k \prec_{\pi} Z_j$ if and only if $Z_k < Z_j$ or it both holds that $Z_k = Z_j$ and $\pi(k) < \pi(j)$.

\subsection{Exact Confidence}
Parameter $m$ influences the resolution of the confidence probability we can achieve. Namely, a probability $p \in (0, 1)$ is {\em admissible} if it can be written in the form of $p = 1 - \nicefrac{q}{m}$, where $q$ is an integer satisfying $0 < q < m$. On the other hand, since both $m$ and $q$ are hyper-parameters, their values are user-chosen. Hence, every rational probability $p \in (0, 1)$ is admissible, by choosing $m$ and $q$ appropriately. Then, a confidence set for an admissible probability $p = p(m,q)$ is
\begin{equation}
A_p \, \doteq \, \left\{ \, \alpha : \mathcal{R}(\alpha) \leq p \, \right\} \, = \,\left\{ \, \alpha : \mathcal{R}_m(\alpha) \leq 1 - \nicefrac{q}{m} \, \right\}.
\end{equation}

One of the main questions is: what kind of stochastic guarantees do such confidence regions have? The following theorem states that they are {\em exact}.

\begin{theorem}
\label{thm-exact}
Under Assumptions \ref{A0}, \ref{A1}, \ref{A2} and \ref{A3}, the coverage probability of the constructed confidence region with respect to the ideal coefficient vector $\alpha^*$ is
\begin{equation}
\Pr \,\big(\, \alpha^* \in A_p \,\big) \, = \,\, p
\,\, = \,\, 1 - \frac{q}{m},
\end{equation}
for any choice of the integer hyper-parameters satisfying\, $0\,<\, q\, < \,m$.
\end{theorem}

\begin{proof}
Following \citet{SPSPaper2ITA}, the core idea is to show that variables
\begin{equation}
Z_0(\alpha^*),\; Z_1(\alpha^*),\; \dots,\; Z_{m-1}(\alpha^*)
\end{equation}
are {\em uniformly ordered}, which means that each ordering of them, with respect to the strict total order $\prec_{\pi}$, has the same probability, that is $1/m!$, formally,
\begin{equation}
\Pr \,\big(\,
Z_{i_0}(\alpha^*) \prec_{\pi} Z_{i_1}(\alpha^*) \prec_{\pi} \dots \prec_{\pi} Z_{i_{m-1}}(\alpha^*) \,\big) \, = \,\, \frac{1}{m!},
\end{equation}
where $(i_0, i_1, \dots, i_{m-1})$ is an arbitrary permutation of $(0,1,\dots, m-1)$. This ordering property is not obvious, since they are not independent, even though we already observed that they are identically distributed (for ideal coefficients).

By definition,  $\alpha^* \in A_p$ if and only if $\mathcal{R}(\alpha^*) \leq 1 - \nicefrac{q}{m}$, i.e., if the reference element, $Z_0(\alpha^*)$ takes one of the positions $1, \dots, m-q$ in the ordering of $\{Z_i(\alpha^*)\}_{i=0}^{m-1}$ variables, w.r.t.\ the strict total order $\prec_{\pi}$. Then, assuming they are uniformly ordered (yet to be shown), we know that $Z_0(\alpha^*)$ takes each position in the ordering with probability exactly $1/m$. Therefore,  for $i \in \{1, \dots, m\}$, we have
\begin{equation}
\Pr\hspace{0.1mm}\Big(\,\mathcal{R}(\alpha^*) \,=\, \frac{i}{m}\;\Big)\, =\, \frac{1}{m},
\end{equation}
from which it follows that $\Pr\bigl(\alpha^* \in A_p\bigr)\, =\, 1 - \nicefrac{q}{m}$ by taking into account that events $\{\,\mathcal{R}(\alpha^*) = \nicefrac{i}{m}\,\}$ and $\{\,\mathcal{R}(\alpha^*) = \nicefrac{j}{m}\,\}$ are disjoint, if $i \neq j$.

In order to show that $\{Z_i(\alpha^*)\}_{i=0}^{m-1}$ are indeed uniformly ordered, we can apply Theorem 2.17 of \citet{KolumbanThesis2016}. Our proposed approach can be interpreted as a variant of a DP method, even though formally the DP ``performance measures'' can depend on the parameters, $\alpha$, the inputs, $x$, and the perturbed outputs, $y^{(i)}$, but not directly on the perturbed residuals. Nevertheless, in our case, $y^{(i)}$ is
\begin{equation}
y^{(i)} \, \doteq \, f_{\alpha}(x)\, +\, G_i (\err(x, y, \alpha)),
\end{equation}
where $f_{\alpha}(x) \doteq [\,f_{\alpha}(x_1), \dots, f_{\alpha}(x_n)  ]\tr$. Then, obviously we can compute the transformed residuals, $G_i (\err(x, y, \alpha))$, from $\alpha$, $x$, and $y^{(i)}$ by using that $G_i (\err(x, y, \alpha)) = y^{(i)} -  f_{\alpha}(x)$. Hence, the DP performance measure in our case is defined as
\begin{equation}
Z(\alpha, x, y^{(i)}) \,\doteq \,\|\, \mathrm{\Psi}(x)\, \bar{g}(x, \alpha, y^{(i)} - f_{\alpha}(x) ) \, \|^2,
\end{equation}
which now fits the DP framework. Our Assumption \ref{A3} ensures that this function is well-defined and, together with Assumption \ref{A1}, it also guarantees that we do not need to compute $\{y^{(i)}\}$ to evaluate the perturbed functions. Our Assumption \ref{A2} directly states that the noise, $\varepsilon$, is invariant under a compact group of transformations, which is a requirement of Theorem 2.17, and we already observed that true errors coincide with the residuals of ideal representations, $\err(x, y, \alpha^*)\, =\, \varepsilon$.
\qed
\end{proof}

Theorem \ref{thm-exact} shows that the confidence region contains the ideal coefficient vector {\em exactly} with probability $p$; this is {\em non-asymptotically} guaranteed, even though the method is {\em distribution-free}. Since $m$ and $q$ are user-chosen (hyper-parameters), the confidence probability is {\em under our control}. The confidence level does not depend on the weighting matrix, but it influences the {\em shape} of the region. Ideally, it should be proportional to the square root of the covariance of the estimate.

\subsection{Quadratic Objectives and Symmetric Noises}
\label{subsec:quadratic}
If we work with convex {\em quadratic} objectives, which have special importance for kernel methods \citep{hofmann2008kernel}, and assume independent and {\em symmetric} noises, we get the Sign-Perturbed Sums (SPS) method \citep{SPSPaper2ITA} as a special case (using the inverse square root of the Hessian as a weighting matrix).

The SPS method uses the classical least-squares (LS) objective function,
\begin{equation}
g(f_{\alpha}, \Dt) \, = \, \|\,z \,-\, \Reg\hspace{0.2mm} \alpha\, \|^2,
\label{eq:LSobj}
\end{equation}
where $z$ denotes the output vector and $\Reg$ is the regressor matrix. Objective \eqref{eq:LSobj} can be seen as the canonical form of many quadratic functions (cf.\ Section \ref{sec-examples}).

When using the SPS method, we make the following assumptions: the noise terms, $\{\varepsilon_i\}$, are independent and have symmetric distributions about zero; and the regressor matrix, $\Reg$, is independent of the noise, it is skinny and has full rank.

For SPS, the reference and the perturbed functions are defined as
\begin{equation}
\label{spseval}
Z_i(\alpha) \, \doteq \, \| \, (\Reg\tr\Reg)^{-\nicefrac{1}{2}} \Reg\tr G_i (z - \Reg\hspace{0.2mm} \alpha)\, \|^2,
\end{equation}
for $i= 0, \dots, m-1$, where $G_i = \mbox{diag}(\sigma_{i, 1}, \dots, \sigma_{i, n})$, for $i \neq 0$, where random variables $\{ \sigma_{i, j} \}$ are i.i.d.\ having {\em Rademacher} distribution, i.e., they take values $+1$ and $-1$ with probability $\nicefrac{1}{2}$ each; and $G_0 = I_n$ is the identity matrix.

It is easy to see that \eqref{spseval} is a special case of construction \eqref{refpert}-\eqref{refpert2}, where $z$ are the outputs and $\Reg$ is computed from the inputs. Besides being {\em exact}, the confidence regions of SPS have additional important properties, such as they are {\em star convex} with the LS estimate, $\widehat{\alpha}$, as a star center \citep{SPSPaper2ITA}. Moreover, they have {\em ellipsoidal outer approximations}, that is there are regions of the form
\begin{equation}
 A^\circ_p \; \doteq \; \Big\{\, \alpha \in \mathbb{R}^n\, :\, (\alpha-\widehat{\alpha})\tr\frac{1}{n}\Reg\tr\Reg(\alpha-\widehat{\alpha})\,\leq\, r \, \Big\},
\end{equation}
where $A_p \, \subseteq  A^\circ_p$ and the radius of the ellipsoid, $r$, can be computed (in polynomial time) by solving semi-definite programming problems \citep{SPSPaper2ITA}.

Hence, for quadratic  problems, the obtained regions are star convex, thus connected, have ellipsoidal outer approximation, thus bounded. These properties ensure that it is easy to work with them. For example, using star convexity and boundedness, we can efficiently explore the region by knowing that every point of it can be reached from the given star center by a line segment inside the region. Moreover, the ellipsoidal outer approximation provides a compact representation.

\section{Applications and Experiments}
\label{sec-examples}
In this section, we show specific applications of the proposed uncertainty quantification (UQ) approach for typical kernel methods, such as LS-SVC, KRR, $\varepsilon$-SVR and kernelized LASSO (KLASSO), in order to demonstrate the framework.

We also present several numerical experiments to illustrate the family of confidence regions we get for various confidence levels. We always set hyper-parameter $m$  to $100$ in the experiments. The figures were constructed by Monte Carlo simulations, i.e., evaluating $1\,000\,000$ random coefficients and drawing the graphs of their induced models with colors indicating their confidence levels.

\vspace{-1mm}
\subsection{Uncertainty Quantification for
Least-Squares Support Vector Classification} We start with a classification problem and consider the Least-Squares Support Vector Classification (LS-SVC) method \citep{suykens1999least}. LS-SVC under the Euclidean distance is known to be equivalent to hard-margin SVC using the Mahalanobis distance \citep{ye2007svm}. It has the advantage that it can be solved by a system of linear equations, in contrast to a quadratic problem.

We assume that $x_k \in \R^d$ and $y_k \in \{+1, -1\}$, for all $k \in \{1, \dots, n\}$, as well as that the slack variables, i.e., the algebraic (signed) distances of the objects from the corresponding margins multiplied by $\|w\|$, are {\em independent} and distributed {\em symmetrically}, for the ideal representation, i.e., the best possible classifier. For simplicity, we consider {\em linear} classification, that is models of the form
\begin{equation}
h_{\alpha} (x_k) \; \doteq \; \mbox{sign}(\, w\tr x_k +  b\,) \,=\, \mbox{sign}(\, \alpha \tr \tilde x_k \,),
\end{equation}
where $x_k$ is an input vector, $\alpha \, \doteq\, [\,b,\, w\tr\,]\tr$ and $\tilde x_k\, \doteq\, [\,1,\, x_k\tr\,]\tr$.

The standard (primal) formulation of (soft-margin) LS-SVM classification is
\begin{align}
\mbox{minimize}\quad& \frac{1}{2}\, w\tr w \,+\, \lambda \sum_{k=1}^n \xi^2_k\\[1mm]
\mbox{subject to}\quad& y_k (w\tr x_k + b) \,=\, 1 - \xi_k\\[-3mm]
\nonumber
\end{align}
for $k= 1, \dots, n$, where $\lambda > 0$ is fixed. Variables $\{\xi_k\}$ are called the {\em slack variables}. The convex quadratic problem above can be rewritten as minimizing
\begin{equation}
g(f_{\alpha}, \Dt) \;\doteq\; \frac{1}{2}\, \|\, B\hspace{0.2mm} \alpha \,\|^2 \,+\, \lambda \, \| \,\one_n - y \odot (X \alpha) \, \|^2,
\label{eq:LSSVMojb}
\end{equation}
where $\one_n \in \R^n$ is the all-one vector, $\odot$ denotes the Hadamard (entrywise) product, $X \doteq [\,\tilde x_1, \dots, \tilde x_n\,]\tr$ and the role of matrix $B$ is to remove the bias, $b$, from $\alpha$, i.e., $B \,\doteq\, \mbox{diag}(0, 1, \dots, 1)$. Note that the reformulated problem \eqref{eq:LSSVMojb} is unconstrained.

Observe that the objective function, $g(f_{\alpha}, \Dt)$, can be further reformulated to take the canonical form of $\|\,z - \Reg\hspace{0.2mm} \alpha\, \|^2$ by using the following $\Reg$ and $z$,
\begin{equation}
\Reg \;=\; \left[
\begin{array}{c}
\sqrt{\lambda}\, (y\one_{d+1}\tr) \odot X \\[1.5mm]
(\nicefrac{1}{\sqrt{2}})\, B
\end{array}
\right]\!,
\qquad\text{and}\qquad
z \;=\; \left[
\begin{array}{c}
\sqrt{\lambda}\, \one_n \\[1mm]
0_{d+1}
\end{array}
\right]\!,
\label{eq:LSSVMreg}
\end{equation}
where
$0_{d+1} \in \R^{d+1}$ is the all-zero vector. Then, we can apply SPS to the obtained (ordinary) LS formulation. However, we should be careful with the transformations, as the new problem has some auxiliary output terms, the zero part of $z$, for which there are no slack variables. The residuals corresponding to that part are not even stochastic, therefore,
the last $d+1$ terms of the residual vector, $z - \Reg\hspace{0.2mm} \alpha$, should not be perturbed. Consequently, the transformation matrices $\{{G}_i\}$ are defined as
\begin{equation}
G_i \,\,\doteq \,\, \left[
\begin{array}{cc}
\,\bar{G}_i &\,\,\, 0\; \\[1mm]
\,0 &\,\,\, I\;
\end{array}
\right]\!,
\label{eq:LSSVMtrsG}
\end{equation}
for $i=0,\dots, m-1$, where $\bar{G}_0 = I_n$ is the identity, and $\bar{G}_i \doteq \mbox{diag}(\sigma_{i,1}, \dots, \sigma_{i,n})$, for $i\neq 0$, where $\{\sigma_{i,j}\}$ are i.i.d.\ Rademacher random variables, as before.

Then, (exact) confidence regions and (honest) ellipsoidal outer approximations can be constructed for the best linear classifier in the domain of coefficients by the SPS method, i.e., \eqref{spseval}, with regressor matrix and output vector as defined in \eqref{eq:LSSVMreg} and transformations as in \eqref{eq:LSSVMtrsG}. The regions will be centered around the LS-SVM classifier, i.e., for all (rational) $p \in (0, 1)$, the coefficients of LS-SVC are contained in $A_p$, assuming it is non-empty. As each coefficient vector uniquely identifies a classifier, the obtained region can be mapped to the model space, as well.

UQ for LS-SVC  is illustrated in Figure \ref{fig:LSsvm}. The observations were generated by adding Laplace noises to the coordinates of the corresponding class centers. The constructed confidence regions are shown both in the coefficient and model spaces, without the bias term, for simplicity. The possibility of constructing (honest) ellipsoidal outer approximations of the (exact) regions is also illustrated.

\begin{figure}[!t]
    \centering
   	\subfigure[UQ for LS-SVC, SPS, model space]{\label{fig:LSSVC-model}\includegraphics[height=53mm]{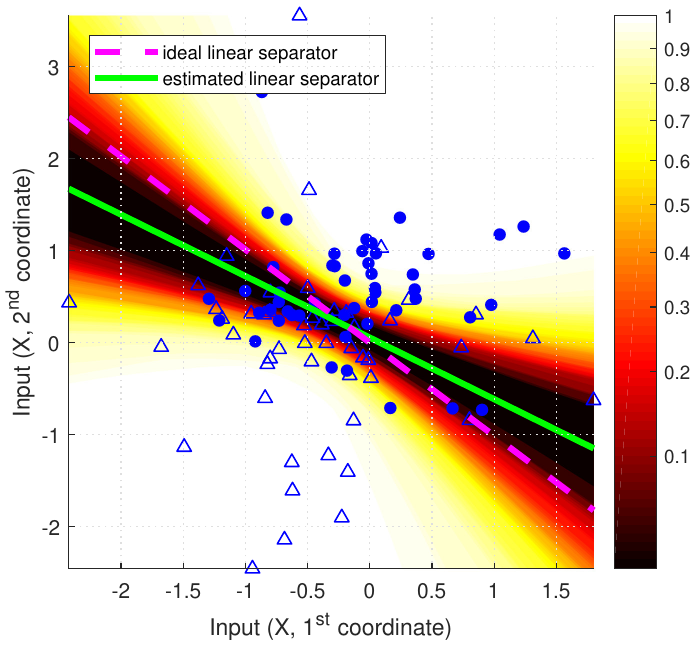}}\quad
 	\subfigure[UQ for LS-SVC, SPS, coefficient space]{\label{fig:LSSVC-coef}\includegraphics[height=53mm]{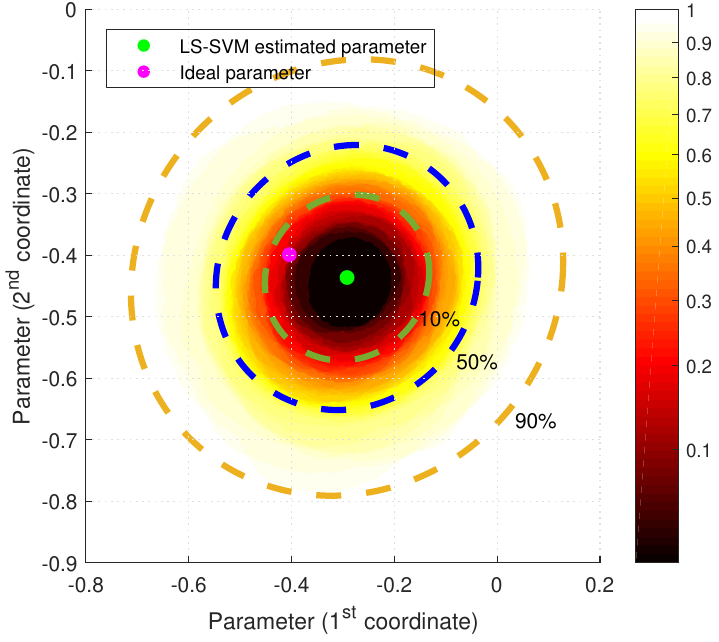}}
	\vspace*{1mm}
	\caption{Exact, non-asymptotic, distribution-free confidence regions for ideal RKHS representations. Parts (a) and (b) present UQ for Least-Squares Support Vector Classification (LS-SVC) with $\lambda = 0.1$ in the model and coefficient spaces, respectively. The ellipsoidal outer approximations of the regions having probabilities $10\,\%$, $50\,\%$ and $90\,\%$ are also presented in the coefficient space.
	There were $n = 100$ observations, $50$ for each class. The centers of the classes were $(0, 0.5)$ and $(-0.5, 0)$. For each observation i.i.d.\ Laplace noises were added to the coordinates of the corresponding centers. The parameters of the noises were $\mu = 0$ (location) and $b = \nicefrac{1}{2}$ (scale). The confidence level of each color can be interpreted by using the scale bars. The regions are increasing, i.e., $A_p \subseteq A_q$ if $p \leq q$, thus, only the smallest levels are shown.}
\label{fig:LSsvm}
\end{figure}

\subsection{Uncertainty Quantification for Kernel Ridge Regression}
\label{special}
Our next example is Kernel Ridge Regression (KRR) which is a kernelized version of Tikhonov regularized LS \citep{shawe2004kernel}. The KRR estimate minimizes a quadratic loss function with a Hilbert space norm regularizer,\vspace{-1.5mm}
\begin{equation}
\hat{f}_{\scriptscriptstyle\text{KRR}} \; \in \; \argmin_{f \in \RKHS}\, \frac{1}{n}\,\sum_{i=1}^n w_i (y_i - f(x_i))^2 \,+\, \lambda\, \| f \|^2_{\RKHS},
\vspace{1mm}
\end{equation}
where $\lambda > 0$, $w_i > 0$, $i=1, \dots, n$, are some a priori given (constant) weights. After using the representer theorem, the objective function can be rewritten as
\begin{align}
\label{krr-obj}
g(f_{\alpha}, \Dt) \, \doteq & \,\,\, \frac{1}{n}\,\sum_{i=1}^n
w_i (y_i - f_{\alpha}(x_i))^2 \,+\, \lambda\, \| f_{\alpha} \|^2_{\RKHS} \, \nonumber \\
= & \,\,\, \frac{1}{n}\,%
\|\, y - f_{\alpha}(x) \,\|^2_{W} \,+\, \lambda\, \| f_{\alpha} \|^2_{\RKHS} \, \nonumber \\[2mm]
= &\,\, \, \frac{1}{n}\,(y - \Ker\hspace{0.2mm} \alpha)\tr W (y - \Ker\hspace{0.2mm} \alpha) \,+\, \lambda\, \alpha\tr \Ker\hspace{0.2mm} \alpha,
\end{align}
where $f_{\alpha}(x) \doteq [\,f_{\alpha}(x_1), \dots, f_{\alpha}(x_n)  ]\tr$, $W \doteq \mbox{diag}(w_1,\dots, w_n)$, $\|v\|^2_W \doteq v\tr W v$, and we used the reproducing property to replace $\| f_{\alpha} \|^2_{\RKHS}$ with a quadratic term.

We can reformulate \eqref{krr-obj} in the canonical form, $\|\,z \,-\, \Reg\hspace{0.2mm} \alpha \,\|^2$, by using
\begin{equation}
\Reg\; =\; \left[
\begin{array}{c}
\;(\nicefrac{1}{\sqrt{n}})\,W^{\frac{1}{2}} \Ker\; \\[1mm]
\sqrt{\lambda}\, \Ker^{\frac{1}{2}}
\end{array}
\right]\!,\qquad\mbox{and}\qquad
z \;=\; \left[
\begin{array}{c}\;
(\nicefrac{1}{\sqrt{n}})\, W^{\frac{1}{2}} y\; \\[1mm]
\;0_n\;
\end{array}
\right]\!,
\label{eq:KRRreg}
\end{equation}
where $W^{\frac{1}{2}}$ and $\Ker^{\frac{1}{2}}$ denote the square roots of matrices  $W$ and $\Ker$, respectively. Note that the square roots exist as these matrices are positive semidefinite.

Then, assuming symmetric and independent measurement noises, formula \eqref{spseval}, with regressor matrix and output vector defined by \eqref{eq:KRRreg}, can be applied to build confidence regions. As in the case of the LS-SVM classifier, the canonical reformulation also contains some auxiliary terms, the zero part of $z$, for which there are no real noise terms, therefore, they should not be perturbed. Thus, we should again use the transformations defined by \eqref{eq:LSSVMtrsG} to get guaranteed confidence regions.

Experiments illustrating the family of (exact, non-asymptotic, distribution-free) confidence regions of KRR with Gaussian kernels and Laplacian measurement noises, and comparing the results with that of support vector regression, are shown in Figure \ref{fig:KRRvsSVM}. The discussion of the comparison is located in Section \ref{sec:SVR}.

\subsection{Uncertainty Quantification for Support Vector Regression}
\label{sec:SVR}
The previous examples were quadratic and therefore, for symmetric noises, their uncertainty could be quantified with SPS. This may no longer be true if we change the applied norms. In this section we study support vector regression, in particular, $\varepsilon$-SVR \citep{hofmann2008kernel,scholkopf2001learning,steinwart2008support}. A well-known advantage of $\varepsilon$-SVR, for example, over KRR, is that it ensures sparse representations through the $\varepsilon$-insensitive loss function. In order to avoid confusion with the true noise vector, $\varepsilon$, we denote the tolerance parameter of the loss function by $\bar{\varepsilon}$. The primal objective function of $\varepsilon$-SVR is defined as\vspace{-1mm}
\begin{equation}
h({f},  \Dt ) \;\doteq \; \frac{1}{2} \, \|\, f \,\|^2_{\RKHS} \,+\, \frac{c}{n} \, \sum_{k=1}^{n}\, \max \{\, 0, |  \left< f, \phi(x_k) \right>_{\RKHS} - y_k \,| -\bar \varepsilon\,\},
\vspace{-1mm}
\label{primalsvr}
\end{equation}
where $f \in \RKHS$, $c>0$, and $\phi(z) \doteq k(z, \cdot)$ is the {\em feature map}. Function \eqref{primalsvr} can be reformulated by applying slack variables, then using standard arguments based on the Lagrangian and the Karush--Kuhn--Tucker (KKT) conditions, we arrive at the Wolfe dual of $\varepsilon$-SVR \citep{scholkopf2001learning}, where we have to maximize
\begin{equation*}
g({f}_{\alpha^+, \alpha^-}, \Dt ) \,= \, y \tr (\alpha^+ - \alpha^-)  \,-
\vspace{-3mm}
\end{equation*}
\begin{equation}
\label{esvrdual}
-\,\frac{1}{2} (\alpha^+ - \alpha^-)  \tr \Ker \, (\alpha^+ - \alpha^-)  - \bar{\varepsilon}\, (\alpha^+ + \alpha^-)\tr \one,
\end{equation}
subject to the (linear) constraints: $\alpha^+, \alpha^- \, \in \, [\,0, \nicefrac{c}{n}\,]^n$ and $(\alpha^+ - \alpha^-) \tr \mathbbm{1} \, = \,0$. One can work directly with the quadratic dual objective, but then the confidence region will be constructed for $\alpha^+, \alpha^-$. Since, $\alpha = \alpha^+ - \alpha^-$, the region could be mapped to a confidence region in the space of coefficient vectors. Alternatively, one can reformulate \eqref{esvrdual} directly for coefficient vector $\alpha$ as\vspace{-1.5mm}
\begin{equation}
\label{esvrl1}
g({f}_{\alpha}, \Dt ) \;= \; y \tr \alpha - \frac{1}{2}\,\alpha \tr \Ker \hspace{0.2mm} \alpha - \bar{\varepsilon}\, \| \alpha \|_1,
\end{equation}
where $\| \cdot \|_1$ is the $1$-norm.
Since \eqref{esvrl1} is concave, a supergradient of it w.r.t.\ $\alpha$ is
\begin{equation}
(\nabla_{\!\alpha}\, g)({f}_{\alpha}, \Dt ) \; = \; y\, - \,\Ker \hspace{0.2mm} \alpha\, -\, \bar{\varepsilon}\,\mbox{sign}(\alpha),
\label{eq:SVRsubgrad}
\end{equation}
where $\mbox{sign}(\cdot)$ denotes the signum function and it is understood component-wise.

\begin{figure*}[!t]
    \centering
	\subfigure[UQ for KRR ($\lambda = 0.1$), SPS]{\label{fig:KRR}\includegraphics[height=54.7mm]{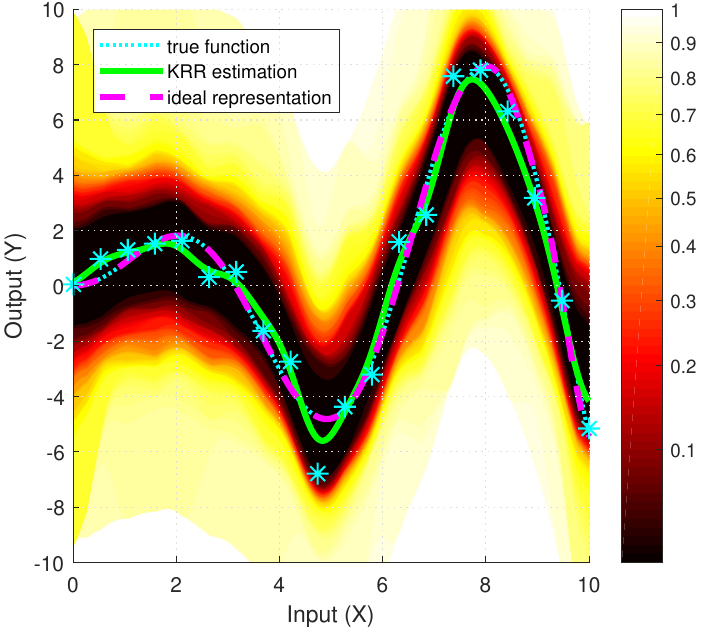}}
 	\subfigure[UQ for $\varepsilon$-SVR ($\bar \varepsilon = 0.2$), sign-changes]{\label{fig:eSVR}\includegraphics[height=54.7mm]{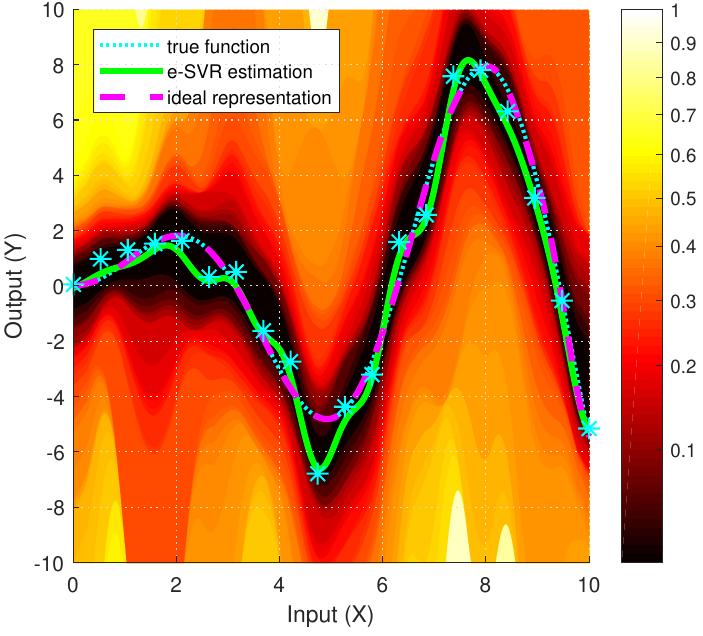}}
	\vspace*{1mm}
	\caption{Exact, non-asymptotic, distribution-free confidence regions for ideal RKHS representations. Parts (a) and (b) show UQ for Kernel Ridge Regression (KRR) with $\lambda = 0.1$ and $\varepsilon$-Support Vector Regression ($\varepsilon$-SVR) with $c = 250$ and $\bar \varepsilon = 0.2$, respectively. The same data  was used for both regression problems, namely, the true function was $f_{*}(x) = x\, \sin(x)$, there were $n = 20$ observations having i.i.d.\ Laplace noise with parameters $\mu = 0$ (location) and $b = \nicefrac{1}{2}$ (scale), and Gaussian kernels were applied with $\sigma = \nicefrac{1}{2}$. Part (a) was built by the Sign-Perturbed Sums (SPS) method, \eqref{spseval}, and formula \eqref{esvr-eval} was used with sign-change matrices for part (b). The confidence level of each color can be interpreted by using the scale bars. The regions are increasing, i.e., $A_p \subseteq A_q$ if $p \leq q$, thus, only the smallest levels are shown.}
\label{fig:KRRvsSVM}
\end{figure*}

Then, building on the supergradient of the dual objective, i.e., \eqref{eq:SVRsubgrad}, reference and perturbed evaluation functions can be defined, for $i=0, \dots, m-1$, as
\begin{equation}
\label{esvr-eval}
Z_i(\alpha) \, \doteq \,  \left\| \,  G_i\,(\hspace{0.2mm}y - \Ker \hspace{0.2mm} \alpha\hspace{0.2mm}) \,-\, \bar{\varepsilon}\,\mbox{sign}(\alpha) \,\right\|^2,
\end{equation}
where $G_0$ is the identity matrix and $G_i$ is a (uniformly chosen) element of the applied compact transformation group, such as a diagonal matrix with $\pm 1$ entries, for symmetric noises (or permutation matrices for exchangeable noises, etc.).

A numerical experiment illustrating the obtained family of confidence regions of the $\varepsilon$-SVR estimate for various confidence levels is shown in Figure \ref{fig:KRRvsSVM}.

The same data sample was used for all regression models, to allow their comparison. The noise affecting the observations was Laplacian, thus heavier-tailed than Gaussian. Since the coefficient space is high-dimensional, and there is a one-to-one correspondence between coefficient vectors and kernel models, the confidence regions are mapped and shown in the model space, i.e., in the RKHS.

Note that it is meaningful to plot the confidence regions even for unknown input values, because the confidence regions are built for the ideal representation, which belongs to the chosen RKHS, unlike the underlying true function.

We can observe that the uncertainty of $\varepsilon$-SVR was higher than that of KRR, which can be explained as the price of using $\varepsilon$-insensitive loss. As the experiments with KLASSO show (cf.\ Figure \ref{fig:KLASSO}), the higher uncertainty of $\varepsilon$-SVR is not simply a consequence of sparse representations, as KLASSO also ensures sparsity. Naturally, the confidence regions are also influenced by the specific choice of hyper-parameters which should be taken into account when the confidence regions are compared.

\subsection{Uncertainty Quantification for Kernelized LASSO}
Our last example covers the LASSO (least absolute shrinkage and selection operator) method, which ensures sparsity via 1-norm regularization. Let us consider the kernelized version of LASSO with objective \citep{wang2007kernel}:
\begin{equation}
\label{klasso-obj}
g(f_{\alpha}, \Dt) \; \doteq \; \frac{1}{2} \, \|\, y - \Ker\hspace{0.2mm} \alpha \,\|^2 \,+ \,\lambda \, \|\,\alpha\,\|_1,
\end{equation}
where $\|\,\cdot\,\|_1$ is the L1 (or Manhattan) norm. Although function \eqref{klasso-obj} cannot be written as $\|\,z \,-\, \Reg\hspace{0.2mm} \alpha \,\|^2$,  the proposed framework, i.e., construction \eqref{refpert}-\eqref{refpert2}, can still be applied. A subgradient of the KLASSO objective \eqref{klasso-obj} is given by
\begin{equation}
\label{klasso-grad}
(\nabla_{\alpha}\, g)(f_{\alpha}, \Dt) \, = \ \Ker (\Ker\hspace{0.2mm} \alpha - y) \,+\, \lambda\,\mbox{sign}(\alpha),
\end{equation}
where the $\mbox{sign}(\cdot)$ function is applied component-wise. Then, using the construction of  \eqref{refpert}-\eqref{refpert2}, the reference and perturbed functions can be defined as
\begin{align}
\label{klasso-eval}
Z_0(\alpha) \; \doteq &\,  \left\| \,  \Ker\hspace{0.2mm}(\Ker\hspace{0.2mm} \alpha - y) \,+\, \lambda\,\mbox{sign}(\alpha) \,\right\|^2,\\[1.5mm]
Z_i(\alpha) \; \doteq &\,  \left\| \,  \Ker\hspace{0.2mm} G_i\,(\Ker\hspace{0.2mm} \alpha - y)\, +\, \lambda\,\mbox{sign}(\alpha) \,\right\|^2,
\end{align}
where $\{G_i\}$ are from a suitable transformation group, e.g., diagonal matrices with Rademacher random variables as diagonal elements for symmetric noises.

Numerical experiments illustrating the confidence regions we get for KLASSO are presented in Figure \ref{fig:KLASSO}. The figure also presents the confidence regions constructed by applying the standard Gaussian Process (GP) regression with estimated parameters. Note that the GP confidence regions are only approximate, namely, they do not come with strict finite-sample guarantees unless the noise is indeed Gaussian. Moreover, during our experiment the noise had a Laplace distribution, which has a heavier tail than Gaussians, therefore even if the true covariance of the noise was known, the confidence regions of GP regression would underestimate the uncertainty of the estimate (would be too optimistic), while the confidence regions of our framework are always non-conservative, independently of the particular distribution of the noise, assuming it has the necessary invariance.

Also note that for our method the noises can even have different (marginal) distributions for each input (e.g., for symmetric noises). Hence, even though the confidence regions generated by GP are smaller than the ones our method builds, the GP regions are imprecise and underestimate the uncertainty of the model, while ours come with strict finite-sample guarantees for a broad class of noises.

\begin{figure*}[!t]
    \centering
 	\subfigure[UQ for KLASSO with Gaussian kernel]{\label{fig:KLASSO-a}\includegraphics[height=54.7mm]{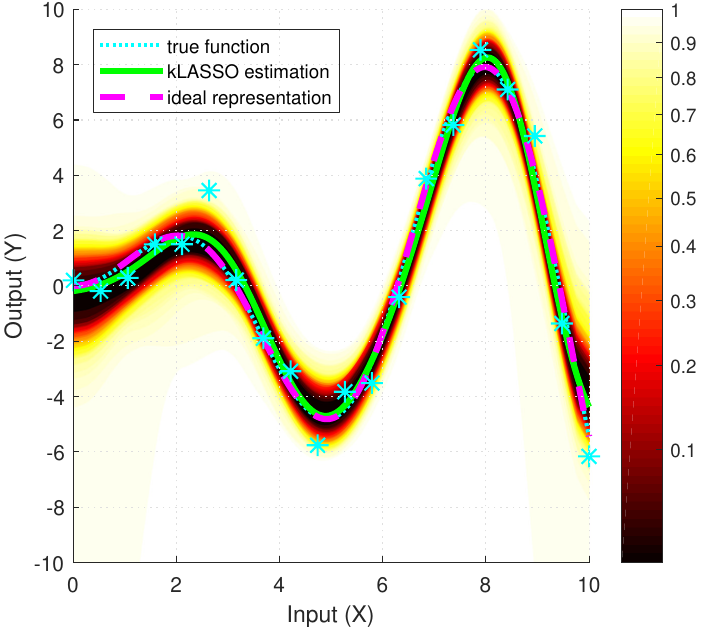}}
 	\subfigure[UQ with Gaussian Process Regression ]{\label{fig:GPR}\includegraphics[height=54.7mm]{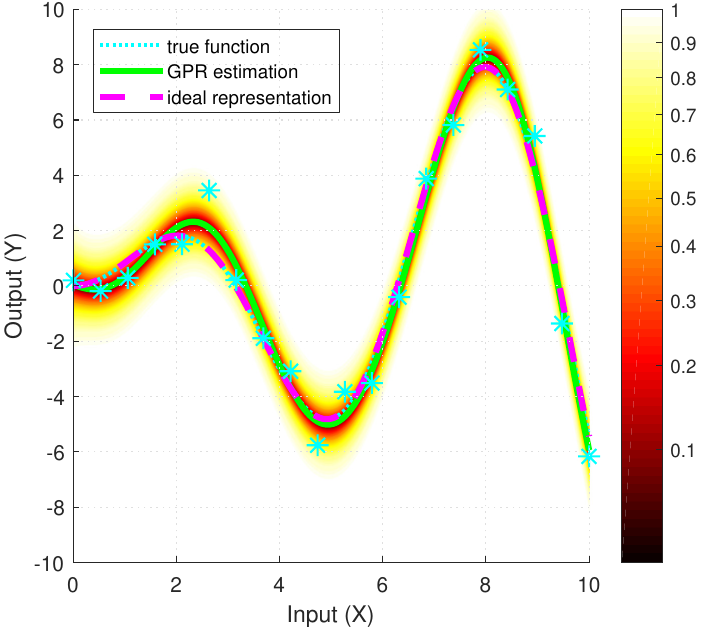}}
	\vspace*{1mm}
	\caption{Exact, non-asymptotic, distribution-free confidence regions for ideal RKHS representations obtained using our framework and approximate confidence regions obtained by Gaussian Process (GP) regression \citep{Rasmussen2006}. Part (a)  shows UQ for Kernelized LASSO with $\lambda = 1$, and part (b) shows UQ with GP. The applied transformations were sign-change matrices. The same data  was used for both regression problems, namely, the true function was $f_{*}(x) = x\, \sin(x)$, there were $n = 20$ observations having i.i.d.\ Laplace noise with parameters $\mu = 0$ (location) and $b = \nicefrac{1}{2}$ (scale), and Gaussian kernels were applied with $\sigma = 1$. The confidence level of each color can be interpreted by using the scale bars. The confidence regions are increasing, i.e., $A_p \subseteq A_q$ if $p \leq q$, therefore, only the smallest levels are shown.}
\label{fig:KLASSO}
\end{figure*}

\section{Conclusions}
In this paper we addressed the problem of quantifying the {\em uncertainty} of kernel estimates by using minimal distributional assumptions. The main aim was to measure the uncertainty of finding the (noise-free) {\em ideal representation} of the underlying (hidden) function at the available inputs. By building on recent developments in finite-sample system identification, we proposed a method that delivers {\em exact}, {\em distribution-free} confidence regions with strong {\em finite-sample guarantees}, based on the knowledge of some mild regularity of the measurement noises. The standard examples of such regularities are {\em exchangeable} or {\em symmetric} noise terms. Note that either of these properties in itself is sufficient for the theory to be applicable.

The needed statistical assumptions are very mild, as for example, no particular (parametric) family of distributions was assumed, {\em no moment assumptions} were made (the noises can be heavy-tailed, and may even have infinite variances); moreover, for the case of symmetric noises, it is allowed that each noise term affecting the observations has a different distribution, i.e., the noise can be {\em nonstationary}.

The core idea of the approach is to evaluate the uncertainty of the estimate by {\em perturbing the residuals} in the {\em gradient} of the objective function. The norms of the (potentially weighted) perturbed gradients are then compared to that of the unperturbed one, and a {\em rank test} is applied for the construction of the region.

The proposed method was also demonstrated on {\em specific examples} of kernel methods. In particular, we showed how to construct exact, non-asymptotic, distribution-free confidence regions for least-squares support vector classification, kernel ridge regression, support vector regression and kernelized LASSO.

Several {\em numerical experiments} were presented, as well, demonstrating that the method provides meaningful regions even for {\em heavier-tailed} than Gaussian (e.g., Laplacian) noises. The figures illustrate whole families of confidence regions for various standard kernel estimates. Ellipsoidal outer approximations are also shown for LS-SVC. Additionally, the method was compared to Gaussian Process (GP) regression, and it was found that although the (approximate) GP confidence regions are smaller in general than our (exact) confidence sets, the GP regions are typically imprecise and they underestimate the real uncertainty of the model.

Our approach to build non-asymptotic, distribution-free, non-conservative confidence regions for kernel methods can be a promising alternative to existing constructions, which typically either build on strong distributional assumptions or on asymptotic theories or only bound the error between the true and empirical risks. As our approach explicitly builds on the constructions of the underlying kernel methods, it can provide {\em new insights} into how the specific methods influence the uncertainty of the estimates, and therefore, besides being vital for risk management, it also has the potential to inspire refinements or new constructions.

There are several open questions about the framework which can facilitate future research directions. For example, finding efficient {\em outer-approximations} for cases when the objective function is not convex quadratic should be addressed. Also the {\em consistency} of the method should be studied to see whether the uncertainty decreases as the sample size tends to infinity. Finally, it would be interesting, as well, to extend the method to (stochastic) {\em dynamical systems} and to formally analyze the {\em size and shape} of the constructed regions in a finite-sample setting.

\begin{acknowledgements}
This research was supported by the National Research, Development and Innovation Office (NKFIH), grant numbers ED\_18-2-2018-0006,  2018-1.2.1-NKP-00008 and KH\_17 125698. The authors are grateful to Algo Car\`e for the valuable discussions.
\end{acknowledgements}

\section*{A. Additional Numerical Experiments}
In this appendix we provide additional numerical experiments supporting the presented framework. The effects of various {\em measurement noises}, {\em kernel functions} and {\em sample sizes} on the obtained (families of) exact, non-asymptotic, distribution-free confidence regions were studied. The true function was always $f_{*}(x) = x\, \sin(x)$ and the inputs were chosen equidistantly from $[\,0, 10\,]$. The regions were evaluated by the same methodology (Monte Carlo simulations) as in Section  \ref{sec-examples}.

\subsection*{A.1 Various Noise Distributions}
First, we investigated how the distribution of the noise affects the regions. In particular, we applied {\em Gaussian}, {\em Laplacian}, {\em Uniform} and {\em Binomial} noises on the outputs of the true function and built the regions for Kernel Ridge Regression (KRR). All noises had zero mean (for the Binomial case the theoretical mean was subtracted from the generated noises), and the parameters of the distributions were set in a way to ensure that all of their variances were the same (i.e., one).

Figure \ref{fig:KRR_Noises} illustrates the obtained families of confidence sets. It can be observed that their shapes and sizes show only small fluctuations indicating that the particular choice of the distribution has a limited effect on the confidence regions (assuming it has zero expectation and we keep the variance of the noise fixed).

\begin{figure*}[!t]
	\vspace*{3mm}
    \centering
	\subfigure[UQ for KRR, {\em Gaussian} noise]{\label{fig:noise-a}\includegraphics[height=54.7mm]{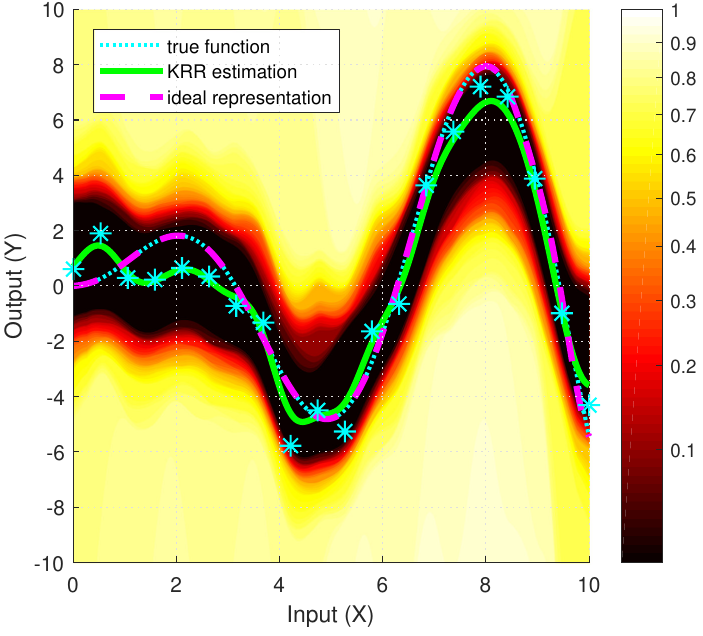}}
	\subfigure[UQ for KRR, {\em Laplace} noise]{\label{fig:noise-b}\includegraphics[height=54.7mm]{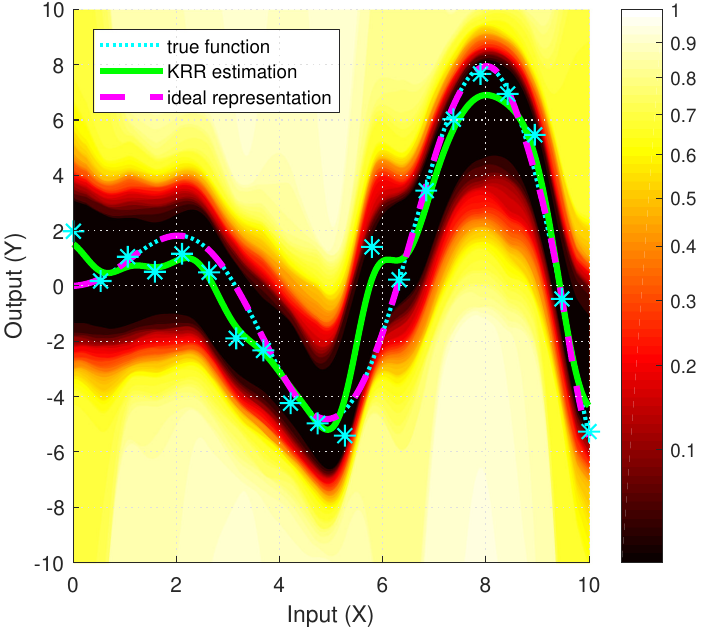}}
	\vspace*{0mm}
	\subfigure[UQ for KRR, {\em Uniform} noise]{\label{fig:noise-c}\includegraphics[height=54.7mm]{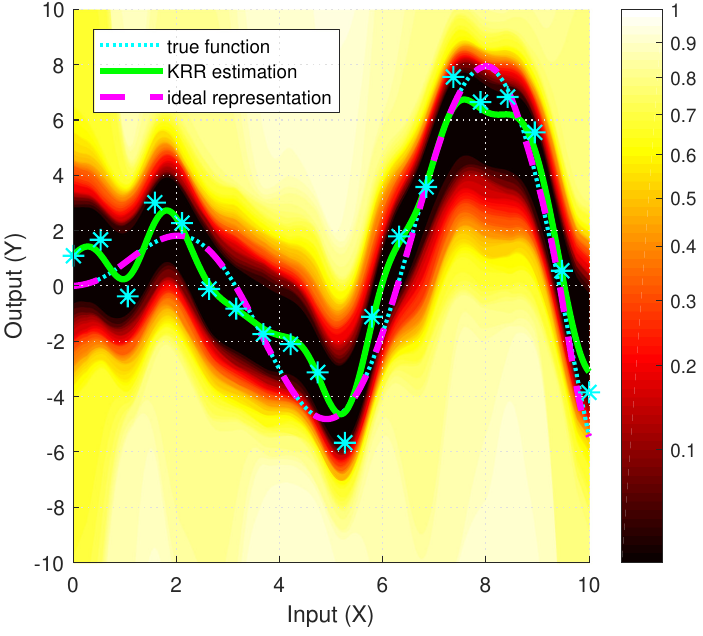}}
   	\subfigure[UQ for KRR, {\em Binomial} noise]{\label{fig:noise-d}\includegraphics[height=54.7mm]{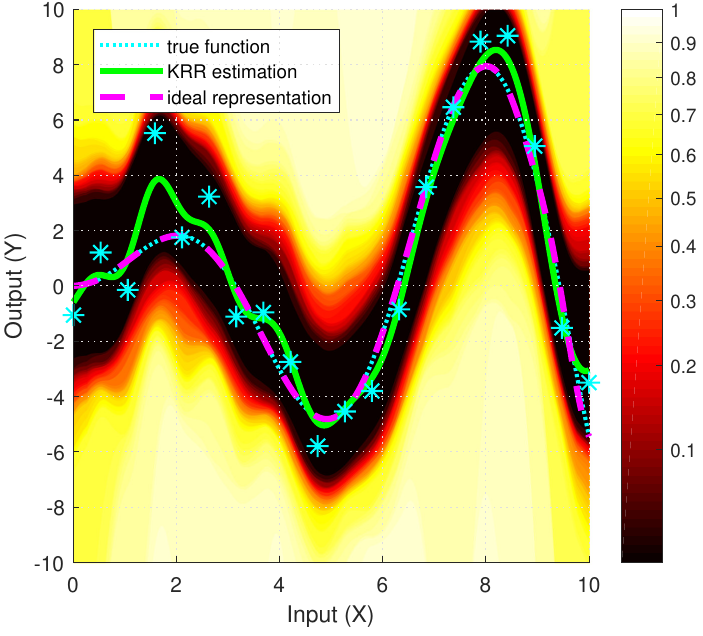}}
	\vspace*{1mm}
	\caption{Exact, non-asymptotic, distribution-free confidence regions for ideal representations w.r.t.\ {\em various noise distributions}. The figure shows UQ for Kernel Ridge Regression (KRR) with $\lambda = 0.1$ and Gaussian kernels with $\sigma = \nicefrac{1}{2}$. Parts (a), (b), (c) and (d) demonstrate the obtained family of confidence regions for i.i.d.\ Gaussian, Laplace, Uniform and Binomial noises, respectively. The parameters of all distributions were set to ensure that each of them has zero mean and unit variance. For the Binomial case, the ``number of trials'' parameter was $20$, and so the ``success probability'' $p$ was set to satisfy $20p(1-p) = 1$ (thus, $p \approx 0.052786$). Then, from each Binomial observation $20p$ was subtracted to ensure zero mean. In all cases  $n = 20$ outputs were measured at equidistant inputs. The Sign-Perturbed Sums (SPS) method was applied to construct the regions, hence, the applied transformations were sign-changes. The confidence levels can be interpreted by using the scale bars. The regions are increasing, i.e., $A_p \subseteq A_q$ if $p \leq q$, therefore, only the smallest levels are shown.}
\label{fig:KRR_Noises}
\end{figure*}

\begin{figure*}[!t]
	\vspace*{3mm}
    \centering
   	\subfigure[UQ for KLASSO, {\em Gaussian} kernel]{\label{fig:kernel-a}\includegraphics[height=54.7mm]{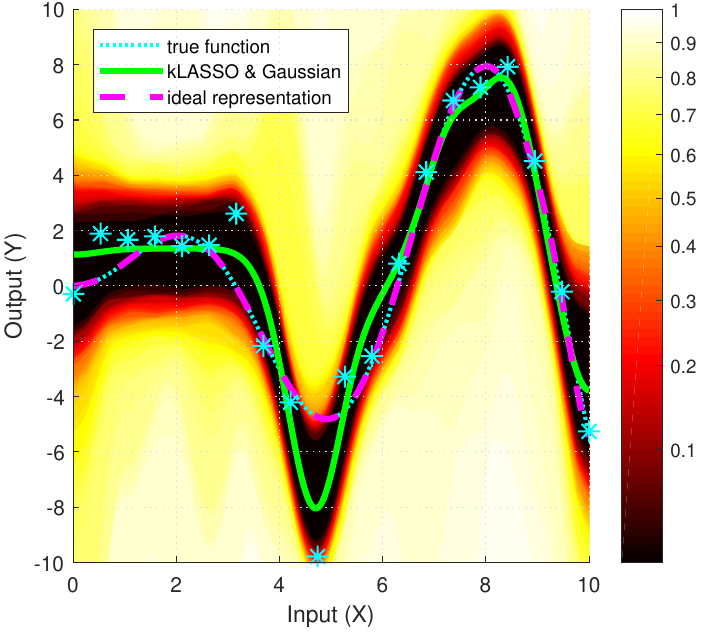}}
	\subfigure[UQ for KLASSO, {\em Laplacian} kernel]{\label{fig:kernel-b}\includegraphics[height=54.7mm]{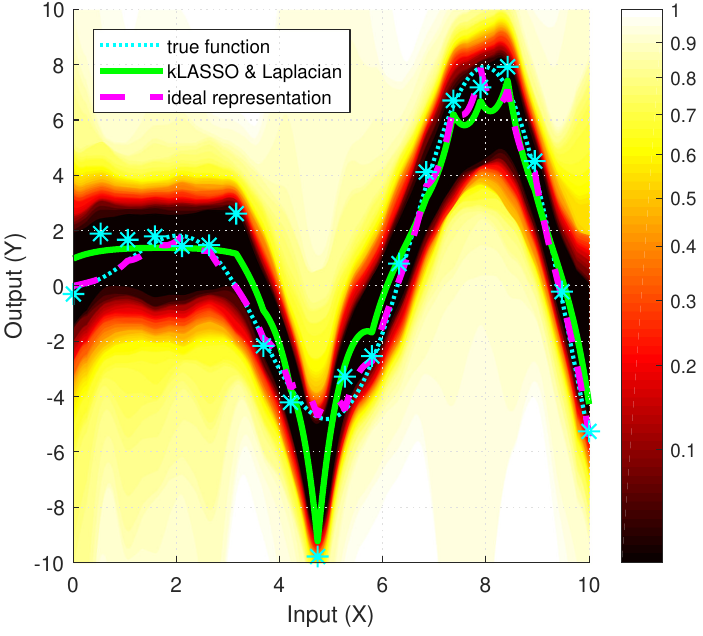}}
		\vspace*{0mm}
	\subfigure[UQ for KLASSO, {\em Parabolic} kernel]{\label{fig:kernel-c}\includegraphics[height=54.7mm]{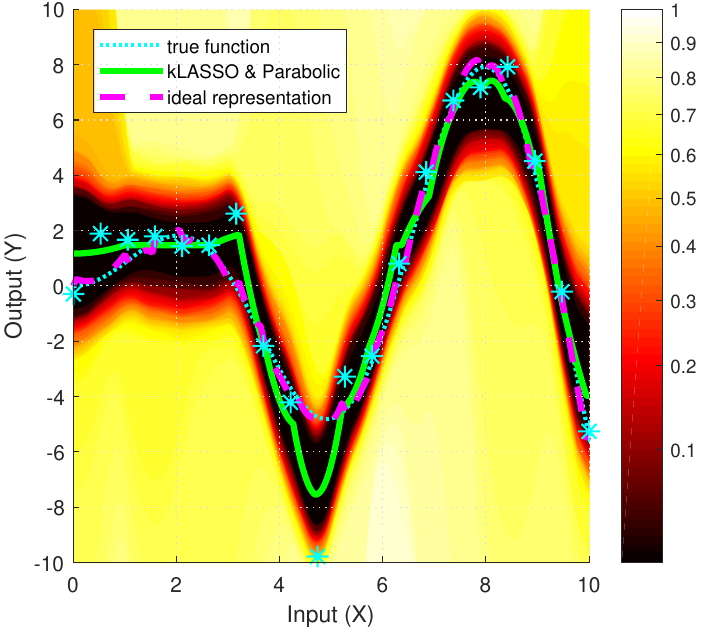}}
	\subfigure[UQ for KLASSO, {\em Rectangular} kernel]{\label{fig:kernel-d}\includegraphics[height=54.7mm]{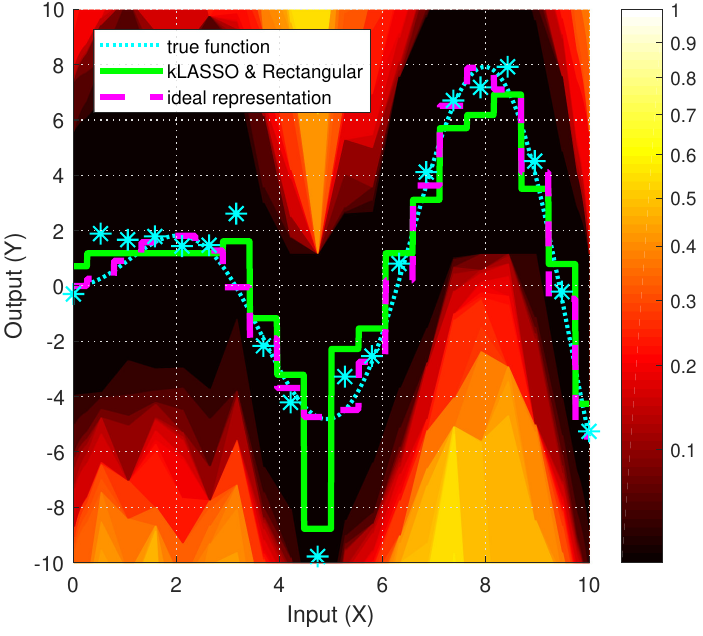}}
	\vspace*{1mm}
	\caption{Exact, non-asymptotic, distribution-free confidence regions for ideal representations w.r.t.\ {\em different kernel functions}. The figure shows UQ for kernelized LASSO with $\lambda = 1$. There were $n = 20$ observations having i.i.d.\ Laplace noises with parameters $\mu = 0$ (location) and $b = \nicefrac{1}{2}$ (scale). Parts (a), (b), (c) and (d) demonstrate the obtained family of confidence regions when using Gaussian, Laplacian, truncated parabolic and rectangular kernels, respectively. For the Gaussian and Laplacian kernels $\sigma = \nicefrac{1}{2}$, for the truncated parabolic kernel $c = 1$, and for the rectangular kernel $c = \nicefrac{1}{38}$. The same data  was used for all regression problems, and the applied transformations were sign-changes. Observe that the Laplacian kernel was more sensitive to the outlier between $4$ and $5$. The obtained regions for the rectangular kernel are much larger than the other regions, indicating a high uncertainty of such an overly localized approach. The confidence levels can be interpreted by using the scale bars. The regions are increasing, i.e., $A_p \subseteq A_q$ if $p \leq q$, therefore, only the smallest levels are shown.}
\label{fig:kLASSO_Kernel}
\end{figure*}

\subsection*{A.2 Different Kernel Functions}
Next, the effect of the applied kernel was studied. Figure \ref{fig:kLASSO_Kernel} illustrates UQ for kernelized LASSO with {\em Gaussian}, {\em Laplacian}, {\em truncated parabolic} ($k(x,y) = \max\{1-c\hspace{0.3mm} \|x-y\|^2,\, 0\}$) and {\em rectangular kernels} ($k(x,y) = \mathbb{I}(\|x - y\| \leq c)$), where the noises were Laplacian. The results show that the choice of the kernel has a significant effect on both the obtained point-estimate (regression model) and the corresponding confidence sets, e.g., the Laplacian kernel was more sensitive to outliers and the regions for the rectangular kernel were much larger than the other ones.

\begin{figure*}[!t]
	\vspace*{3mm}
    \centering
   	\subfigure[UQ for KLASSO, $n = 10$]{\label{fig:size-a}\includegraphics[height=54.7mm]{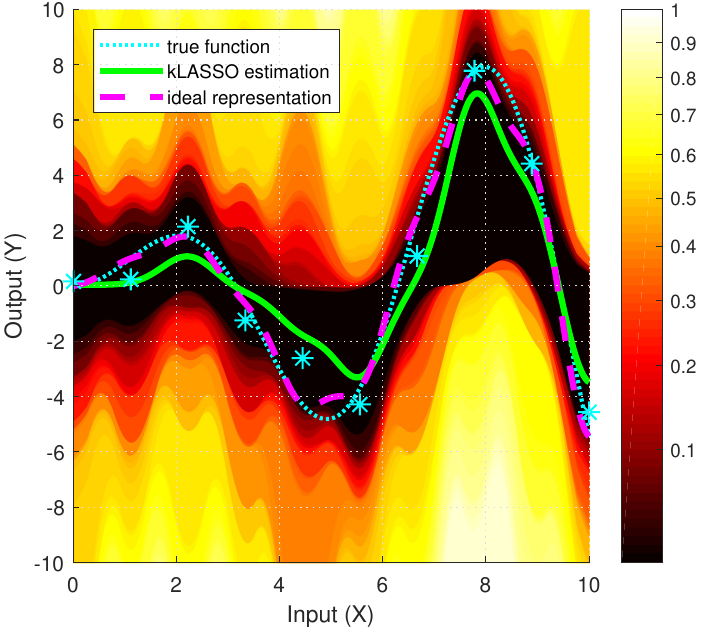}}
	\subfigure[UQ for KLASSO, $n = 20$]{\label{fig:size-b}\includegraphics[height=54.7mm]{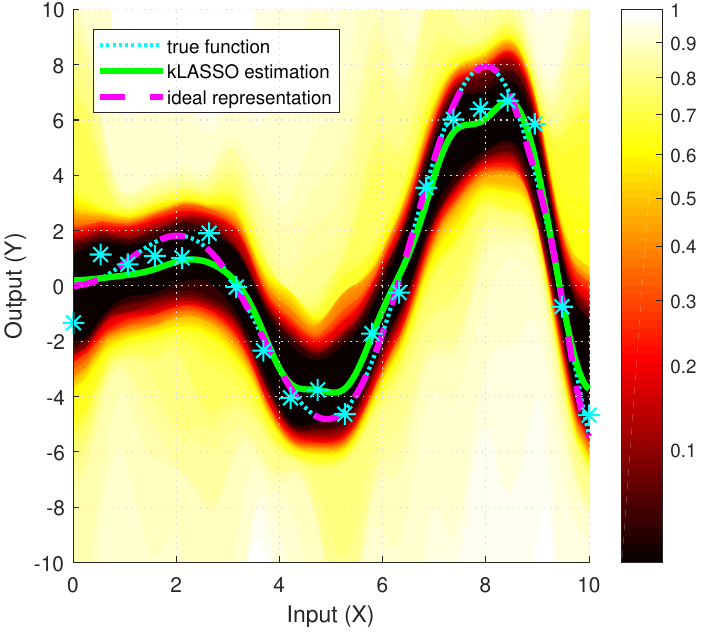}}
    \vspace*{0mm}
	\subfigure[UQ for KLASSO, $n = 50$]{\label{fig:size-c}\includegraphics[height=54.7mm]{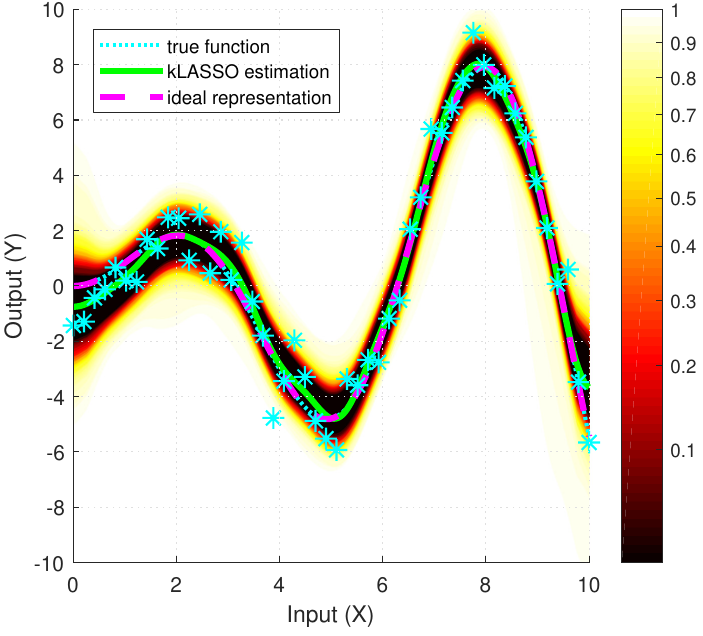}}
	\subfigure[UQ for KLASSO, $n = 100$]{\label{fig:size-d}\includegraphics[height=54.7mm]{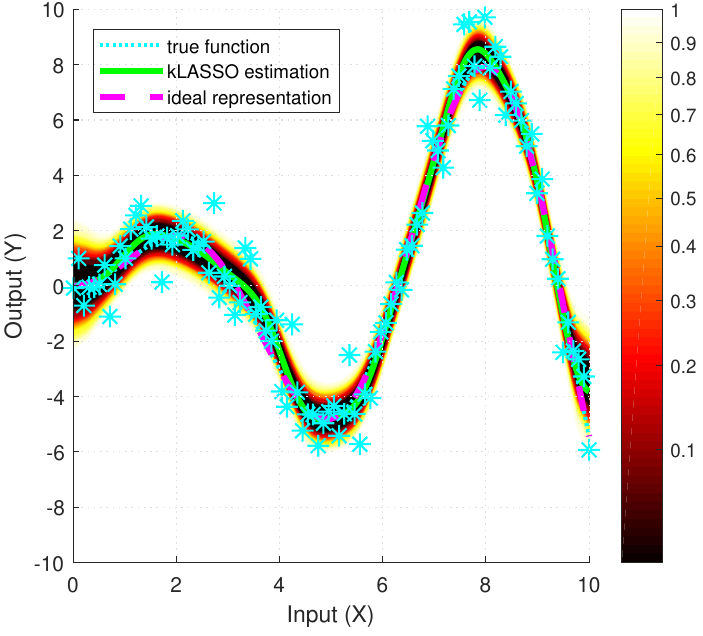}}
	\vspace*{1mm}
	\caption{Exact, non-asymptotic, distribution-free confidence regions for ideal representations w.r.t.\ {\em increasing sample sizes}. The figure shows UQ for kernelized LASSO with $\lambda = 1$ and using Gaussian kernels with $\sigma = \nicefrac{1}{2}$.	 The observations had i.i.d.\ Laplace noises with parameters $\mu = 0$ (location) and $b = \nicefrac{1}{2}$ (scale). Parts (a), (b), (c) and (d) demonstrate the obtained family of confidence regions when using samples of size $n = 10, 20, 50,$ and $100$, respectively. The applied transformations were sign-changes. Observe that the confidence regions shrink around the ideal representations, despite the number of coefficients also increases with the sample size. This is indicative of the phenomenon that the regions have a consistency property. This may be especially true if we apply a universal kernel, such as the Gaussian one, for which the ideal representations can approximate arbitrarily well any continuous function on a compact domain. The confidence levels can be interpreted by using the scale bars. The regions are increasing, i.e., $A_p \subseteq A_q$ if $p \leq q$, therefore, only the smallest levels are shown.}
\label{fig:kLASSO_SampleSize}
\end{figure*}

\subsection*{A.3 Increasing the Sample Size}
Finally, we have experimented with KLASSO to see how increasing the sample size affects the confidence regions. The measurement noises were Laplacian (hence heavier-tailed than Gaussian), and the applied {\em sample sizes} were $n=10, 20, 50,$ and  $100$. The results are shown in Figure \ref{fig:kLASSO_SampleSize} and are indicative of the phenomenon that the confidence regions, and hence the uncertainties, shrink as the sample size tends to infinity, even though the number of coefficients increases with the sample size. This experiment supports that the approach is ``consistent'',  nevertheless, we leave the theoretical investigation of this phenomenon for further study.

\bibliographystyle{spbasic}
\bibliography{Dist-Free-UQ-for-Kernel-Methods}

\end{document}